\documentclass{article}

\usepackage{arxiv}
\usepackage[utf8]{inputenc} 
\usepackage[T1]{fontenc}    
\usepackage{hyperref}       
\usepackage{url}            
\usepackage{booktabs}       
\usepackage{amsfonts}       
\usepackage{nicefrac}       
\usepackage{microtype}      
\usepackage{lipsum}
\usepackage{graphicx}
\graphicspath{ {./images/} }

\usepackage{amsmath,amsfonts,bm}

\def\eqref#1{equation~\ref{#1}}

\def\1{\bm{1}}

\DeclareMathAlphabet{\mathsfit}{\encodingdefault}{\sfdefault}{m}{sl}
\SetMathAlphabet{\mathsfit}{bold}{\encodingdefault}{\sfdefault}{bx}{n}

\usepackage{natbib}
\usepackage{graphicx}
\usepackage{booktabs}
\usepackage{multirow}
\usepackage{amsmath,amssymb,amsfonts}
\usepackage{amsthm}
\usepackage{mathrsfs}
\usepackage[title]{appendix}
\usepackage{xcolor}
\usepackage{textcomp}
\usepackage{manyfoot}
\usepackage{algorithm}
\usepackage{algorithmicx}
\usepackage{algpseudocode}
\usepackage{listings}
\usepackage{float}
\usepackage{adjustbox,lipsum}
\usepackage{silence}
\ErrorsOff*

\newtheorem{theorem}{Theorem}
\newtheorem{proposition}[theorem]{Proposition}%

\title{Curvature-Aware Radius Shrinkage for Adaptive Nearest Neighbor Classification}

\author{
 Alexandre Luís Magalh\~aes Levada\\
  Computing Department\\
  Federal University of S\~ao Carlos\\
  13565-905, S\~ao Carlos-SP, Brazil\\
  \texttt{alexandre.levada@ufscar.br} \\
}

\begin{document}
\maketitle
\begin{abstract}
	Nearest neighbor classification relies fundamentally on a notion of locality,
	yet conventional $k$-NN imposes the same neighborhood cardinality throughout
	the feature space. This assumption can be inadequate for heterogeneous data
	whose local geometry varies substantially across the underlying data manifold.
	We introduce Curvature-Aware Radius Shrinkage for Adaptive Nearest Neighbor
	Classification (CARSANN), a geometry-driven framework that adapts the
	\emph{spatial support} of each neighborhood according to local geometric
	complexity. CARSANN first estimates the intrinsic dimensionality using TwoNN
	and constructs an intrinsic representation through principal component
	analysis. Local mean curvature is then estimated through a shape-operator-based
	formulation and used to control the neighborhood scale: highly curved regions
	receive stronger radius shrinkage, whereas approximately flat regions retain
	broader spatial support. This formulation introduces a distinct mechanism of
	adaptivity that differs fundamentally from methods that modify only the number
	of neighbors or the local metric. Extensive experiments on more than 70
	real-world classification datasets from OpenML demonstrate that CARSANN
	substantially improves upon standard $k$-NN and is competitive with
	discriminative and curvature-based adaptive nearest-neighbor methods. In a
	controlled comparison using the same base neighborhood size
	$k=k_{\mathrm{base}}$, CARSANN achieves higher balanced accuracy on 40 of 45
	datasets, with the mean balanced accuracy increasing from $0.6506$ to
	$0.7528$. The advantage persists when standard $k$-NN is configured with the
	fixed value $k=5$, for which the mean balanced accuracy increases from
	$0.6547$ to $0.7568$. Friedman and Nemenyi tests further confirm that these
	improvements are statistically significant. Overall, the results provide
	evidence that local manifold curvature can serve as an effective geometric
	control variable for determining the spatial extent of neighborhood-based
	evidence, establishing radius adaptation as a complementary paradigm to
	cardinality-based adaptive nearest neighbor classification.
\end{abstract}

\section{Introduction}\label{sec1}

Nearest neighbor methods constitute one of the most fundamental families of nonparametric learning algorithms. Their appeal stems from a simple yet powerful principle: observations that are close to each other in an appropriate feature space are expected to exhibit similar class-conditional behavior. The classical nearest neighbor rule, introduced in the seminal work of Cover and Hart~\cite{CoverHart1967}, requires no explicit parametric model of the underlying data distribution and can therefore accommodate highly nonlinear decision structures. The $k$-nearest neighbors ($k$-NN) classifier extends this principle by aggregating the labels of the $k$ closest training observations, providing a simple and interpretable mechanism for controlling the locality of the classification rule. Despite its simplicity, nearest neighbor classification has remained an important reference method in pattern recognition and statistical learning.

A fundamental limitation of standard $k$-NN, however, is that the neighborhood size is governed by a single global parameter $k$. This assumption implicitly imposes the same notion of locality throughout the entire feature space, even when the underlying data distribution is strongly heterogeneous. A small value of $k$ produces highly localized decision rules and can preserve fine-scale structures, but at the cost of increased variance and sensitivity to noise and outliers. Conversely, a large value of $k$ reduces variance by aggregating more observations, but may introduce substantial bias by averaging across heterogeneous regions and oversmoothing class boundaries. Thus, the choice of $k$ directly affects the bias--variance tradeoff, the smoothness of the resulting decision boundary, and the robustness of the classifier to local sampling fluctuations. More fundamentally, a globally fixed $k$ does not account for the fact that different regions of the feature space may require different spatial scales for reliable neighborhood-based inference.

This limitation becomes particularly relevant when data exhibit nonuniform density, nonlinear structure, class imbalance, or locally varying geometric complexity. In such situations, the number of observations required to obtain a meaningful local estimate may vary substantially from one region to another. A neighborhood that is appropriate in a locally homogeneous region may be excessively large near a complex decision structure, whereas a neighborhood that is sufficiently conservative near a difficult region may contain too little statistical evidence in a relatively homogeneous region. Consequently, the problem of determining an appropriate neighborhood is not merely a matter of selecting a globally optimal value of $k$, but rather of defining a notion of locality that adapts to the local structure of the data.

This observation has motivated a broad family of adaptive nearest neighbor methods. Early approaches attempted to determine a different value of $k$ for each query point or training observation. For example, Sun and Huang~\cite{SunHuang2010} proposed an adaptive $k$-nearest neighbor algorithm in which the neighborhood size is determined locally rather than fixed globally. Such methods recognize that the optimal neighborhood cardinality may vary throughout the feature space and therefore seek to replace the global parameter $k$ by a locally determined quantity. More generally, adaptive nearest neighbor methods can be interpreted as attempts to establish a spatially varying scale of locality, thereby reducing the mismatch between a globally fixed neighborhood and heterogeneous data distributions.

A more geometric and discriminative perspective was introduced by Hastie and Tibshirani through the Discriminant Adaptive Nearest Neighbor (DANN) classifier~\cite{HastieTibshirani1996}. Rather than merely changing the number of neighbors, DANN estimates a local discriminative metric using linear discriminant information. The resulting metric modifies the geometry of the neighborhood by shrinking it in directions approximately orthogonal to local decision boundaries and elongating it along directions parallel to those boundaries. Thus, DANN can be viewed as adapting the local metric tensor according to class-discriminative information. This formulation represents an important conceptual transition from selecting a neighborhood cardinality to modifying the geometry of the neighborhood itself.

More recent approaches have explored the intrinsic geometry of the data as a source of information for neighborhood adaptation. In particular, Levada, Nielsen, and Haddad~\cite{LevadaNielsenHaddad2024} proposed the curvature-adaptive $k$-nearest neighbor classifier, denoted $kK$-NN, which estimates local Gaussian curvature through a shape-operator-based construction and uses this geometric quantity to determine the local number of neighbors. The underlying intuition is that low-curvature regions admit a larger approximately linear neighborhood, whereas high-curvature regions should be treated more locally because the tangent approximation becomes less accurate. The resulting classifier therefore adapts the neighborhood cardinality according to local manifold geometry. Experiments reported by the authors show that curvature-based adaptation can improve balanced accuracy over conventional $k$-NN and other adaptive nearest neighbor approaches, particularly in settings with limited training data~\cite{LevadaNielsenHaddad2024}.

The progression from classical $k$-NN to adaptive nearest neighbor methods suggests an important distinction between different mechanisms for controlling locality. Standard $k$-NN uses a globally fixed neighborhood cardinality. Classical adaptive $k$-NN methods allow the cardinality to vary across observations. DANN goes one step further by modifying the local metric according to discriminative structure. The $kK$-NN method introduces differential geometry into this framework by adapting the neighborhood cardinality according to local curvature. These approaches are closely related, but they address different aspects of the neighborhood-selection problem.

In this work, we argue that an alternative and complementary perspective is to adapt not the number of neighbors itself, but the \emph{spatial extent} of the neighborhood. This distinction is central to the proposed method. Rather than asking how many observations should be included around a query point, we ask how far from that point local evidence should be allowed to influence the classification decision. This leads to a radius-based formulation in which the effective support of the neighborhood is explicitly contracted or preserved according to local geometric complexity. We refer to this mechanism as \emph{radius shrinkage}.

The proposed approach is motivated by the manifold hypothesis, according to which high-dimensional observations may concentrate near a lower-dimensional structure embedded in the ambient feature space. In high-dimensional datasets, the ambient dimension may therefore substantially overestimate the number of degrees of freedom required to describe the local data geometry. This distinction is particularly relevant for curvature estimation, because reliable local geometric estimation becomes increasingly difficult as the ambient dimension increases. Consequently, before estimating local curvature, we seek to identify a lower-dimensional representation that better reflects the intrinsic structure of the data.

To this end, the proposed framework follows the following geometric pipeline:
\begin{equation}
	\boxed{
		\begin{array}{c}
			\text{ambient dimension}
			\rightarrow
			\text{intrinsic dimension}
			\rightarrow
			\text{intrinsic representation}
			\\[2mm]
			\rightarrow\;
			\text{local curvature}
			\rightarrow
			\text{adaptive radius}
			\rightarrow
			\text{classification}
	\end{array}}
\end{equation}

The first stage estimates the intrinsic dimensionality of the dataset using the Two-NN estimator introduced by Facco et al.~\cite{FaccoEtAl2017}. Two-NN exploits the statistics of the distances to the first and second nearest neighbors and provides an estimate of the intrinsic dimension using minimal local neighborhood information. This estimate is subsequently used to determine the dimensionality of the representation on which local geometric quantities are computed. In particular, principal component analysis (PCA) is employed to obtain an intrinsic-dimensional representation of the observations. This step is not introduced merely as a computational preprocessing operation; rather, it defines the geometric space in which local curvature is subsequently estimated.

Let $d$ denote the estimated intrinsic dimension. The original observations
$\mathbf{x}_i \in \mathbb{R}^{m}$, where $m$ is the ambient dimension, are mapped to
\begin{equation}
	\mathbf{z}_i \in \mathbb{R}^{d},
	\qquad
	d \ll m,
\end{equation}
whenever the data exhibit a substantially lower intrinsic dimension. Local geometric information is then extracted from the resulting representation. Specifically, the local curvature at each observation is estimated using a shape-operator-based formulation. The curvature provides a scalar characterization of the local geometric complexity of the data support and is subsequently used to determine the appropriate spatial scale of the neighborhood.

The key hypothesis underlying our approach is that the appropriate spatial scale of local aggregation should depend on the geometric complexity of the region surrounding the query point. In a locally flat region, the tangent space provides a good approximation of the data manifold over a relatively broad neighborhood. Aggregating evidence from a larger spatial region can therefore provide a more stable estimate of the local class distribution. In contrast, in a highly curved region, the tangent approximation deteriorates more rapidly as the neighborhood expands. Including increasingly distant observations may consequently mix samples belonging to geometrically distinct local configurations and may blur relevant class distinctions. We therefore associate larger curvature with stronger radius shrinkage and smaller curvature with weaker shrinkage.

Let $r_i^{(0)}$ denote a reference radius associated with the query point $\mathbf{x}_i$. In our formulation, this reference radius is obtained from a baseline nearest-neighbor neighborhood, with the initial neighborhood size determined as $k_0 = \log_2(n)$, where $n$ is the number of training observations. A curvature-dependent shrinkage function is then applied to obtain an effective radius
\begin{equation}
	r_i = s(\kappa_i)\, r_i^{(0)},
\end{equation}
where $\kappa_i$ denotes the local curvature and $s(\cdot)$ is a monotone nonincreasing shrinkage function. Consequently, high-curvature observations receive smaller effective neighborhoods, whereas low-curvature observations retain a radius close to the reference neighborhood. Classification is finally performed using the training observations contained within the curvature-adapted radius.

This formulation leads to a useful geometric taxonomy of adaptive nearest neighbor methods. Standard $k$-NN determines a globally fixed neighborhood cardinality and therefore imposes a global notion of locality. Generic adaptive $k$-NN methods allow the neighborhood cardinality $k$ to vary spatially. DANN adapts the local metric tensor using discriminant information, thereby modifying the shape and orientation of the neighborhood~\cite{HastieTibshirani1996}. The $kK$-NN classifier adapts the neighborhood cardinality using local curvature~\cite{LevadaNielsenHaddad2024}. In contrast, the proposed CARSANN classifier adapts the \emph{radius of the neighborhood support} using a geometric measure of local complexity. This distinction can be summarized as
\begin{equation}
	\begin{array}{rcl}
		\text{$k$-NN} &:& \text{global neighborhood cardinality},\\[2mm]
		\text{adaptive $k$-NN} &:& \text{local neighborhood cardinality},\\[2mm]
		\text{DANN} &:& \text{local metric geometry from discriminative information},\\[2mm]
		\text{$kK$-NN} &:& \text{local neighborhood cardinality from curvature},\\[2mm]
		\text{CARSANN} &:& \text{local neighborhood radius from curvature}.
	\end{array}
\end{equation}

The distinction between cardinality adaptation and radius adaptation is particularly important. Although the two mechanisms are related, they need not produce equivalent neighborhoods. Changing $k$ determines how many observations are retained, whereas changing the radius determines the spatial extent within which observations are considered relevant. In heterogeneous datasets, these two notions of locality can differ substantially because the same number of observations may occupy very different spatial volumes depending on local density. The proposed method therefore provides a complementary mechanism for adaptive neighborhood selection: rather than directly optimizing the number of neighbors, it controls the spatial scale over which local evidence is aggregated.

The contributions of this work are therefore fourfold. First, we introduce a curvature-aware radius shrinkage mechanism for nearest neighbor classification that explicitly adapts the spatial support of local decisions according to geometric complexity. Second, we propose an intrinsic-dimension-aware preprocessing strategy in which TwoNN is used to estimate the effective dimensionality of the data and PCA subsequently constructs an intrinsic representation for local geometric estimation. Third, we establish a conceptual distinction between curvature-based adaptation of neighborhood cardinality, as in $kK$-NN, and curvature-based adaptation of neighborhood radius, as proposed here, thereby introducing a new geometric degree of freedom for adaptive nearest neighbor classification. Fourth, we conduct an extensive empirical evaluation against standard $k$-NN and representative adaptive and geometry-aware alternatives, demonstrating that radius adaptation can provide substantial and statistically significant improvements in classification performance across heterogeneous real-world datasets.

The remainder of this paper is organized as follows. Section 2 presents some related work on adaptive nearest neighbor classification and geometric approaches to neighborhood selection. Section 3 introduces differential geometry fundamentals and intrinsic-dimensional representation based on the TwoNN algorithm. Section 4 presents the local curvature estimation procedure in details. Section 5 describes the proposed curvature-aware radius shrinkage adaptive nearest neighbor classifier (CARSANN). Section 6 presents the experimental protocol and discusses the obtained results. Finally, Section 7 concludes the paper and discusses directions for future research.

\section{Related Work}
\label{sec3}

The limitations of using a globally fixed neighborhood size in $k$-NN classification have motivated extensive research on adaptive nearest neighbor methods. The central idea is to allow the neighborhood to vary according to local properties of the data, thereby providing a more flexible notion of locality than that offered by a single global value of $k$~\citep{3,5,4}. Early approaches primarily focused on determining a sample-dependent neighborhood cardinality. Sun and Huang~\citep{5}, for example, proposed AdaNN, which estimates an individual $k$ for each training observation based on the minimum number of neighbors required for its correct classification. Other approaches have considered local density and neighborhood structure as criteria for selecting $k$~\citep{4,6}, while Maeng et al.~\citep{3} proposed MANNC by combining adaptive and modified nearest neighbor strategies.

A related line of research has investigated more sophisticated mechanisms for determining the local neighborhood size, including heuristic and optimization-based strategies~\citep{2,7,8}. Ougiaroglou et al.~\citep{2} proposed early termination criteria to avoid computing unnecessary neighbors once sufficient evidence for classification has been obtained. Onyezewe et al.~\citep{7} employed simulated annealing to efficiently search for a near-optimal $k$ for each test observation, while Zhang et al.~\citep{8} introduced BLA-KNN, which combines adaptive neighbor weighting with class-aware regularization. These approaches illustrate that neighborhood adaptation can be formulated not only through local data statistics but also as a computational or optimization problem.

More recent studies have incorporated geometric information into adaptive nearest neighbor selection. In particular, Levada et al.~\citep{1} proposed the curvature-adaptive $kK$-NN classifier, which estimates the local curvature of the data manifold using a shape-operator-based formulation and uses this quantity to determine the neighborhood cardinality. The underlying intuition is that larger neighborhoods are appropriate in approximately flat regions, where the tangent space provides a reliable local approximation, whereas highly curved regions require more localized neighborhoods. This work represents an important transition from heuristic adaptation toward geometry-driven neighborhood selection.

Adaptive nearest neighbor methods have also been successfully incorporated into a variety of application-specific frameworks, including geo-spatial classification~\citep{4}, seabed sediment classification~\citep{9}, Twitter spam detection~\citep{10}, and mechanical fault diagnosis~\citep{11}. These studies demonstrate the flexibility of adaptive neighborhood construction and its ability to accommodate heterogeneous data distributions and application-specific structures. However, their primary objective remains the adaptation of neighborhood cardinality, either directly or through weighting and graph construction.

Scalability is another important issue in nearest neighbor methods, particularly for large or high-dimensional datasets. Existing approaches have addressed this challenge through early termination and indexing strategies~\citep{2}, data pruning and partitioning~\citep{12}, and parallel or distributed implementations~\citep{13,14}. In particular, distributed $k$-NN methods based on frameworks such as Apache Spark have been developed to enable nearest neighbor classification on large-scale datasets~\citep{14}. Although these techniques substantially improve computational efficiency, they primarily address the cost of neighborhood construction rather than the geometric definition of an appropriate neighborhood.

Despite the substantial progress in adaptive nearest neighbor classification, most existing approaches focus on determining \emph{how many} neighbors should be considered. Even geometry-aware methods such as $kK$-NN use local curvature to adapt the neighborhood cardinality. In contrast, the present work considers a complementary question: \emph{how far should the neighborhood extend from a query point?} This distinction is particularly relevant in heterogeneous data distributions, where the same number of neighbors can correspond to substantially different spatial scales. The proposed CARSANN addresses this gap by using local curvature to adapt the \emph{radius} defining the neighborhood support. Thus, rather than directly selecting a spatially varying $k$, CARSANN controls the extent of local aggregation according to the geometric complexity of the data, providing a complementary mechanism for adaptive nearest neighbor classification.

\section{Differential Geometry Concepts}
\label{sec:differential_geometry}

This section briefly introduces the differential-geometric concepts required to formulate the proposed curvature-aware neighborhood adaptation method. The presentation focuses on the local geometry of smooth manifolds embedded in Euclidean spaces, with particular emphasis on tangent spaces, metric structure, the second fundamental form, the shape operator, and mean curvature \cite{DGApp,Manfredo,Tristan}.

\subsection{Manifolds and Local Coordinates}

A $d$-dimensional smooth manifold $\mathcal{M}$ embedded in an $m$-dimensional Euclidean space $\mathbb{R}^{m}$ is a set for which every point admits a local neighborhood that is smoothly equivalent to an open subset of $\mathbb{R}^{d}$. In other words, although $\mathcal{M}$ may be embedded in a high-dimensional ambient space, its local degrees of freedom are described by only $d$ coordinates.

More formally, a local parametrization of $\mathcal{M}$ around a point $\mathbf{x}\in\mathcal{M}$ is a smooth mapping
\begin{equation}
	\boldsymbol{\phi}: U \subset \mathbb{R}^{d}
	\longrightarrow
	\mathcal{M}\subset\mathbb{R}^{m},
\end{equation}
where $U$ is an open subset of $\mathbb{R}^{d}$ and $\boldsymbol{\phi}$ is a local coordinate map. Let
\begin{equation}
	\mathbf{x}
	=
	\boldsymbol{\phi}(u^1,\ldots,u^d)
\end{equation}
denote a point on the manifold. The partial derivatives
\begin{equation}
	\mathbf{e}_i
	=
	\frac{\partial \boldsymbol{\phi}}
	{\partial u^i},
	\qquad
	i=1,\ldots,d,
\end{equation}
span the local tangent space of the manifold at $\mathbf{x}$. The manifold hypothesis provides a natural motivation for this formulation in machine learning: high-dimensional observations may lie close to a lower-dimensional manifold, such that the intrinsic dimension $d$ is substantially smaller than the ambient dimension $m$. Consequently, geometric quantities computed in the intrinsic space can provide information about the local structure of the data that is obscured by the ambient representation.

\subsection{Tangent Space}

The tangent space of $\mathcal{M}$ at $\mathbf{x}$, denoted by
$T_{\mathbf{x}}\mathcal{M}$, is the $d$-dimensional linear space containing all tangent vectors to curves on $\mathcal{M}$ passing through $\mathbf{x}$. Given a smooth curve
\begin{equation}
	\boldsymbol{\gamma}:(-\epsilon,\epsilon)\rightarrow\mathcal{M},
	\qquad
	\boldsymbol{\gamma}(0)=\mathbf{x},
\end{equation}
its velocity vector
\begin{equation}
	\mathbf{v}
	=
	\left.
	\frac{d\boldsymbol{\gamma}(t)}{dt}
	\right|_{t=0}
\end{equation}
belongs to $T_{\mathbf{x}}\mathcal{M}$. For a local parametrization $\boldsymbol{\phi}$, the tangent space can be expressed as
\begin{equation}
	T_{\mathbf{x}}\mathcal{M}
	=
	\operatorname{span}
	\left\{
	\frac{\partial\boldsymbol{\phi}}{\partial u^1},
	\ldots,
	\frac{\partial\boldsymbol{\phi}}{\partial u^d}
	\right\}.
\end{equation}

The tangent space provides a first-order approximation of the manifold around $\mathbf{x}$. In particular, for a sufficiently small displacement, the manifold can be approximated by its tangent space. This approximation becomes increasingly inaccurate as the neighborhood extends into regions where the manifold exhibits stronger curvature. This observation is central to the proposed approach, since local curvature is subsequently used to determine an appropriate spatial scale for neighborhood aggregation.

\subsection{Metric Tensor and the First Fundamental Form}

To measure lengths and angles between tangent vectors, the manifold must be equipped with a metric. For an embedded manifold in Euclidean space, the metric is naturally induced by the standard inner product of $\mathbb{R}^{m}$. Given two tangent vectors $\mathbf{v},\mathbf{w}\in T_{\mathbf{x}}\mathcal{M}$, the metric tensor $g$ is defined as
\begin{equation}
	g_{\mathbf{x}}(\mathbf{v},\mathbf{w})
	=
	\langle \mathbf{v},\mathbf{w}\rangle.
\end{equation}

In local coordinates, the metric tensor is represented by the first fundamental form
\begin{equation}
	I_{\mathbf{x}}(\mathbf{v},\mathbf{w})
	=
	\sum_{i,j=1}^{d}
	g_{ij}\,v^i w^j,
\end{equation}
where
\begin{equation}
	g_{ij}
	=
	\left\langle
	\frac{\partial\boldsymbol{\phi}}{\partial u^i},
	\frac{\partial\boldsymbol{\phi}}{\partial u^j}
	\right\rangle.
\end{equation}

The matrix
\begin{equation}
	G(\mathbf{x})
	=
	\left[g_{ij}\right]_{i,j=1}^{d}
\end{equation}
is the local metric tensor. It determines the intrinsic notion of length and angle on the manifold. In particular, the squared norm of a tangent vector
$\mathbf{v}=\sum_i v^i\mathbf{e}_i$ is given by
\begin{equation}
	\|\mathbf{v}\|_g^2
	=
	\mathbf{v}^{\mathsf T}
	G(\mathbf{x})
	\mathbf{v}.
\end{equation}

Thus, the first fundamental form characterizes the intrinsic geometry of the manifold, whereas the quantities introduced next describe how the manifold is embedded and bends within its ambient space.

\subsection{Normal Space and the Second Fundamental Form}

The normal space $N_{\mathbf{x}}\mathcal{M}$ is the orthogonal complement of the tangent space in the ambient space:
\begin{equation}
	N_{\mathbf{x}}\mathcal{M}
	=
	\left(T_{\mathbf{x}}\mathcal{M}\right)^{\perp}.
\end{equation}

For a hypersurface, i.e., a manifold with $d=m-1$, the normal space is one-dimensional. In this case, a unit normal vector $\mathbf{n}(\mathbf{x})$ can be selected such that
\begin{equation}
	\left\|
	\mathbf{n}(\mathbf{x})
	\right\|=1,
	\qquad
	\left\langle
	\mathbf{n}(\mathbf{x}),
	\mathbf{v}
	\right\rangle=0
\end{equation}
for every $\mathbf{v}\in T_{\mathbf{x}}\mathcal{M}$. The second fundamental form describes the extrinsic curvature of the manifold by measuring the variation of the normal direction along tangent directions. For tangent vectors $\mathbf{v},\mathbf{w}\in T_{\mathbf{x}}\mathcal{M}$, the second fundamental form is defined as
\begin{equation}
	II_{\mathbf{x}}(\mathbf{v},\mathbf{w})
	=
	\left\langle
	\mathbf{n},
	D_{\mathbf{v}}\mathbf{w}
	\right\rangle,
\end{equation}
where $D_{\mathbf{v}}$ denotes the directional derivative in the ambient Euclidean space. Equivalently, using the variation of the normal vector, the second fundamental form can be written as
\begin{equation}
	II_{\mathbf{x}}(\mathbf{v},\mathbf{w})
	=
	-
	\left\langle
	D_{\mathbf{v}}\mathbf{n},
	\mathbf{w}
	\right\rangle.
\end{equation}

In local coordinates, the coefficients of the second fundamental form are
\begin{equation}
	b_{ij}
	=
	\left\langle
	\frac{\partial^2\boldsymbol{\phi}}
	{\partial u^i\partial u^j},
	\mathbf{n}
	\right\rangle.
\end{equation}

The second fundamental form therefore characterizes how rapidly the manifold departs from its tangent space. While the first fundamental form describes the local metric structure, the second fundamental form captures the local bending of the embedded manifold.

\subsection{Shape Operator}

The relationship between the first and second fundamental forms is conveniently expressed through the shape operator, also known as the Weingarten map. For a hypersurface, the shape operator at $\mathbf{x}$ is the linear operator
\begin{equation}
	S_{\mathbf{x}}:
	T_{\mathbf{x}}\mathcal{M}
	\rightarrow
	T_{\mathbf{x}}\mathcal{M}
\end{equation}
defined by
\begin{equation}
	S_{\mathbf{x}}(\mathbf{v})
	=
	-
	D_{\mathbf{v}}\mathbf{n}.
\end{equation}

The shape operator is self-adjoint with respect to the metric induced on the tangent space. Consequently, it admits an orthonormal eigenbasis. Let
\begin{equation}
	S_{\mathbf{x}}\mathbf{e}_i
	=
	\kappa_i\mathbf{e}_i,
	\qquad
	i=1,\ldots,d,
\end{equation}
where $\kappa_i$ are the principal curvatures and $\mathbf{e}_i$ are the corresponding principal directions. In matrix form, the shape operator can be obtained from the first and second fundamental forms as
\begin{equation}
	S
	=
	G^{-1}B,
\end{equation}
where
\begin{equation}
	G=[g_{ij}]
	\qquad\text{and}\qquad
	B=[b_{ij}]
\end{equation}
are the matrices associated with the first and second fundamental forms, respectively. The eigenvalues of $S$ provide a local spectral characterization of the manifold's curvature. In particular, large magnitudes of the principal curvatures indicate directions in which the manifold bends strongly, whereas small values indicate directions in which the manifold is locally closer to its tangent approximation. Figure \ref{fig:shape} illustrates the intuition behind the shape operator.

\begin{figure}[h]
	\centering
	\includegraphics[scale=0.25]{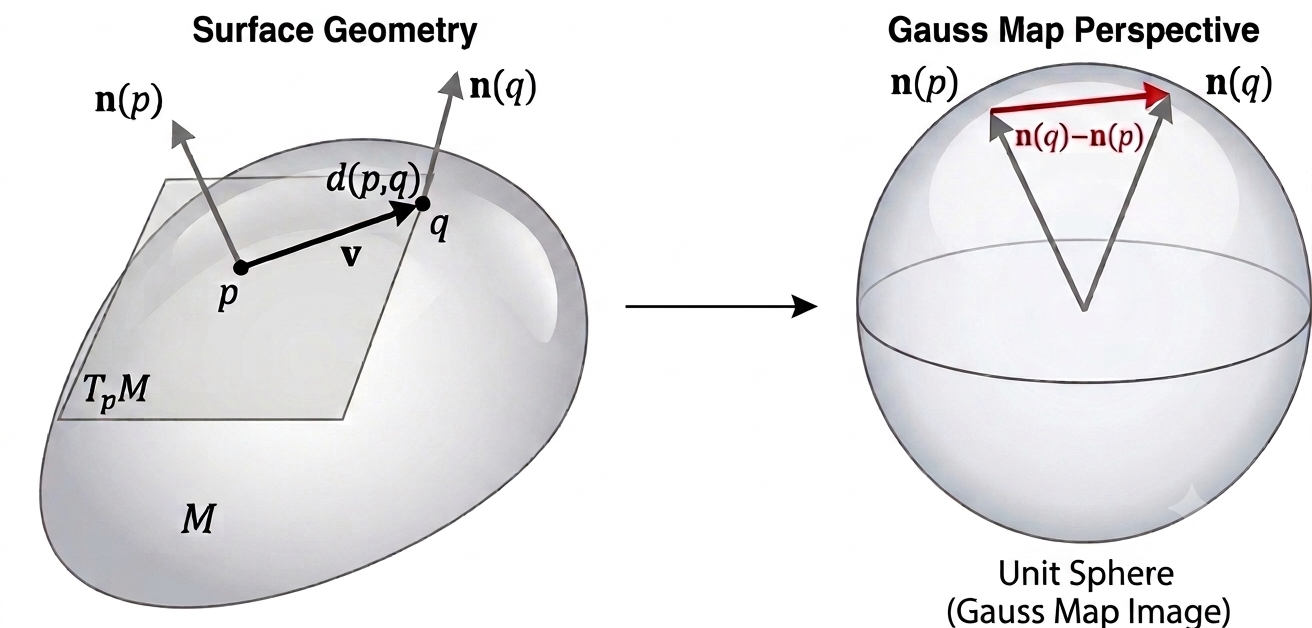}
	\caption{Geometric interpretation of the shape operator. As one moves along a tangent direction $\mathbf{v}$, the shape operator measures the directional variation of the unit normal vector, providing a quantitative description of the local bending of the manifold.}
	\label{fig:shape}
\end{figure}

\subsection{Mean Curvature}

The mean curvature is a scalar summary of the principal curvatures and is defined as the normalized trace of the shape operator:
\begin{equation}
	H(\mathbf{x})
	=
	\frac{1}{d}
	\operatorname{tr}
	\left(
	S_{\mathbf{x}}
	\right)
	=
	\frac{1}{d}
	\sum_{i=1}^{d}\kappa_i.
\end{equation}

Equivalently, in terms of the first and second fundamental forms,
\begin{equation}
	H(\mathbf{x})
	=
	\frac{1}{d}
	\operatorname{tr}
	\left(
	G^{-1}B
	\right).
\end{equation}

The mean curvature provides a compact measure of the local bending of the manifold. Regions with small $|H(\mathbf{x})|$ are locally closer to a minimal or approximately flat configuration, whereas larger values of $|H(\mathbf{x})|$ indicate stronger average bending across the principal directions. For the purposes of adaptive nearest neighbor classification, mean curvature provides a natural geometric quantity with which to characterize local complexity. A neighborhood located in a region of low curvature can generally extend farther while remaining compatible with the local tangent approximation. Conversely, in a region of high curvature, increasing the neighborhood radius causes the local tangent approximation to deteriorate more rapidly. This motivates the use of curvature as a control variable for the spatial scale of neighborhood aggregation.

\subsection{Geometric Interpretation for Neighborhood Adaptation}

The concepts introduced above establish the geometric foundation of the proposed CARSANN method. Given a data manifold $\mathcal{M}$ and a point $\mathbf{x}\in\mathcal{M}$, the tangent space $T_{\mathbf{x}}\mathcal{M}$ provides a first-order local approximation, while the metric tensor defines distances and angles within this local representation. The second fundamental form and the shape operator characterize how the manifold bends away from its tangent space, with the eigenvalues of the shape operator providing the principal curvatures. This hierarchy can be summarized as

\begin{equation}
	\boxed{\text{metric}
	\;\longrightarrow\;
	\text{tangent geometry}
	\;\longrightarrow\;
	\text{extrinsic curvature}
	\;\longrightarrow\;
	\text{local geometric complexity}}
\end{equation}

The proposed method exploits the final quantity in this hierarchy to determine the appropriate spatial scale for local classification. Specifically, local curvature is used to shrink a reference neighborhood in geometrically complex regions while preserving a broader neighborhood in approximately flat regions. This establishes the connection between differential geometry and the adaptive radius mechanism developed in the next section.

\begin{figure}[H]
	\centering
	\includegraphics[scale=0.25]{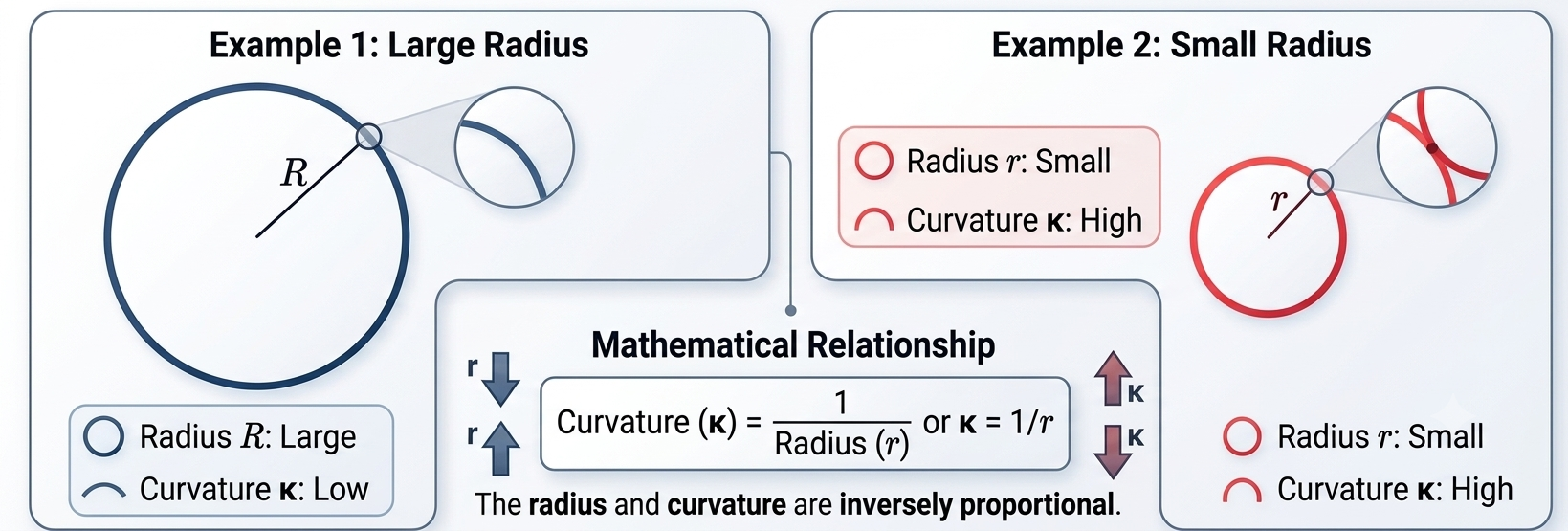}
	\caption{Geometric motivation of the CARSANN radius-shrinkage mechanism. Regions of higher local curvature exhibit a faster departure from the tangent-space approximation and therefore receive a smaller effective neighborhood radius. Approximately flat regions can retain a larger spatial support.}
	\label{fig:radius}
\end{figure}

%

\subsection{Intrinsic Dimension and the TwoNN Algorithm}
\label{sec:twonn}

The differential-geometric formulation presupposes that the intrinsic dimension $d$ of the data manifold $\mathcal{M}$ is known: it fixes the number of local coordinates in the parametrization $\phi: \mathbb{R}^d \to \mathcal{M}$ of Eq.~(5), the rank of
the tangent space $T_x\mathcal{M}$, and, in Section~\ref{sec:quadratic-feature-map},
the number of independent quadratic features used in the local shape-operator
estimator. In practice, $d$ is unknown and must itself be estimated from the
data before any of the subsequent geometric machinery can be applied. This
subsection formalizes the notion of intrinsic dimension and describes TwoNN
\cite{FaccoEtAl2017}, the minimal-neighborhood estimator adopted by
CARSANN to obtain $\hat d$ from the ambient observations
$\{x_1, \dots, x_n\} \subset \mathbb{R}^m$, prior to constructing the
intrinsic-dimensional representation $z_i \in \mathbb{R}^{\hat d}$ of
Eq.~(2).

\paragraph{Intrinsic dimension.}
Under the manifold hypothesis, the observations $x_i \in \mathbb{R}^m$ are
assumed to lie on (or near) a $d$-dimensional submanifold
$\mathcal{M} \subset \mathbb{R}^m$, with $d \ll m$ in the regime of
interest. The intrinsic dimension $d$ is the smallest number of local
coordinates needed to parametrize $\mathcal{M}$ in a neighborhood of any of
its points (Eq.~(5)); equivalently, it is the dimension of the tangent
space $T_x\mathcal{M}$ at (almost) every $x \in \mathcal{M}$. Unlike the
ambient dimension $m$, which is fixed by the choice of feature
representation and therefore only an upper bound on the true number of
degrees of freedom of the data, $d$ is an intrinsic geometric property of
the data-generating manifold. Estimating $d$ reliably is a prerequisite for
every subsequent step of the proposed pipeline: it determines both the
dimensionality of the representation on which local curvature is estimated
(Section~\ref{sec:quadratic-feature-map}) and the size of the neighborhoods CARSANN uses for curvature estimation and classification.

\paragraph{The TwoNN estimator.}
Among the many approaches proposed for estimating $d$ -- ranging from
global spectral methods such as PCA-based dimension estimation to local
maximum-likelihood estimators \cite{levina2004maximum} -- CARSANN adopts
TwoNN \cite{FaccoEtAl2017}, which estimates $d$ using only the
distances from each point to its first and second nearest neighbors. This
minimality is deliberate: because the estimator only requires information
from an infinitesimally small neighborhood around each point, it is
comparatively robust to the density inhomogeneity and curvature that
typically bias estimators relying on larger, fixed-size neighborhoods
\cite{FaccoEtAl2017}, and it avoids introducing an additional
neighborhood-size hyperparameter at a stage of the pipeline that precedes,
and is a prerequisite for, the neighborhood-size choices of
Section~\ref{sec:dual-scale}.

Let $r_1(x_i)$ and $r_2(x_i)$ denote the (Euclidean) distances from $x_i$ to
its first and second nearest neighbors in $\{x_1, \dots, x_n\}$, and define
the per-point ratio
\begin{equation}
	\mu_i = \frac{r_2(x_i)}{r_1(x_i)} \geq 1.
	\label{eq:mu-ratio}
\end{equation}
TwoNN rests on the following distributional result, which we derive here
for completeness.

\begin{proposition}[Pareto law of nearest-neighbor ratios]
	\label{prop:pareto}
	Suppose the observations in a neighborhood of $x$ are generated by a
	homogeneous Poisson point process of (locally constant) rate $\rho$ on the
	$d$-dimensional manifold $\mathcal{M}$, restricted to the scale spanned by
	the first two nearest neighbors of $x$ (local density homogeneity). Then
	$\mu(x) = r_2(x)/r_1(x)$ has cumulative distribution function
	\begin{equation}
		F(\mu) = 1 - \mu^{-d}, \qquad \mu \geq 1,
		\label{eq:pareto-cdf}
	\end{equation}
	i.e., $\mu(x)$ follows a Pareto distribution with unit scale and shape
	parameter equal to the intrinsic dimension $d$, independently of the local
	density $\rho$.
\end{proposition}

\begin{proof}
	For a homogeneous Poisson process of rate $\rho$ on a $d$-dimensional
	space, let $V_k = \omega_d\, r_k(x)^d$ denote the volume of the ball of
	radius $r_k(x)$ enclosing exactly the $k$ nearest neighbors of $x$, where
	$\omega_d$ is the volume of the unit ball in $\mathbb{R}^d$. It is a
	standard property of Poisson processes that the rescaled enclosed volumes
	$T_k = \rho\, V_k$ are the arrival times of a unit-rate Poisson process on
	$[0,\infty)$, so that $T_1 \sim \mathrm{Exp}(1)$ and the increment
	$\Delta T = T_2 - T_1 \sim \mathrm{Exp}(1)$, independently of $T_1$, by the
	memoryless property of the Poisson process. Since $\omega_d$ and $\rho$
	cancel in the ratio,
	\begin{equation}
		\mu^d = \left(\frac{r_2(x)}{r_1(x)}\right)^{d} = \frac{V_2}{V_1} = \frac{T_2}{T_1} = 1 + \frac{\Delta T}{T_1} = 1 + Q,
	\end{equation}
	where $Q = \Delta T / T_1$ is the ratio of two independent
	$\mathrm{Exp}(1)$ random variables. The ratio of two i.i.d.\ unit-rate
	exponential random variables has cumulative distribution function
	$P(Q \le q) = q/(1+q)$ for $q \ge 0$, so that
	\begin{equation}
		P(\mu^d > s) = P(Q > s - 1) = \frac{1}{s}, \qquad s \geq 1,
	\end{equation}
	i.e., $\mu^d$ follows a standard Pareto distribution, $P(\mu^d \le s) = 1 -
	1/s$. Setting $s = \mu^d$ and substituting back yields
	$P(\mu \le t) = P(\mu^d \le t^d) = 1 - t^{-d}$ for $t \ge 1$, which is
	Eq.~\eqref{eq:pareto-cdf}.
\end{proof}

Differentiating Eq.~\eqref{eq:pareto-cdf} gives the density $f(\mu) = d\,
\mu^{-d-1}$, $\mu \geq 1$: the distribution of $\mu$ depends on the
intrinsic dimension $d$ alone, with no dependence on the (unknown, and
generally spatially varying) local density $\rho$ or on the ambient
dimension $m$. This is the key property that makes TwoNN practical: $d$ can
be recovered from the empirical distribution of the ratios $\{\mu_i\}$
without any density normalization.

\paragraph{Estimation by log-log regression.}
Rearranging Eq.~\eqref{eq:pareto-cdf},
\begin{equation}
	-\log\big(1 - F(\mu)\big) = d \, \log(\mu),
	\label{eq:pareto-linear}
\end{equation}
which is a line of slope $d$ through the origin in the $(\log \mu,\,
-\log(1-F(\mu)))$ plane. Facco et al.~\cite{FaccoEtAl2017} propose
estimating $d$ by replacing $F$ with its empirical counterpart and fitting
Eq.~\eqref{eq:pareto-linear} by least squares. Concretely, let $\mu_{(1)}
\le \mu_{(2)} \le \dots \le \mu_{(n)}$ be the order statistics of
$\{\mu_i\}_{i=1}^n$, and define the empirical cumulative distribution
\begin{equation}
	F^{\mathrm{emp}}(\mu_{(j)}) = \frac{j-1}{n}, \qquad j = 1, \dots, n,
	\label{eq:empirical-cdf}
\end{equation}
so that the estimation reduces to a linear regression on the point set
$\big\{\big(x_j, y_j\big)\big\}_{j=1}^n$ with $x_j = \log(\mu_{(j)})$ and
$y_j = \log\!\big(1 - F^{\mathrm{emp}}(\mu_{(j)})\big)$, for which
Eq.~\eqref{eq:pareto-linear} predicts the relation $y_j = -d\, x_j$. Rather
than constraining the fit to pass through the origin, CARSANN follows the
common practical relaxation of fitting an unconstrained affine model
\begin{equation}
	y_j = a\, x_j + b
	\label{eq:affine-fit}
\end{equation}
by ordinary least squares and recovering the dimension estimate as $\hat d
= -a$, discarding the intercept $b$. This relaxation absorbs small,
systematic deviations from Eq.~\eqref{eq:pareto-linear} near $\mu = 1$
without biasing the slope estimate, at the cost of one additional degree of
freedom in the regression.

Because Proposition~\ref{prop:pareto} relies on local density homogeneity
over the (small) neighborhood spanned by the first two nearest neighbors, it
is systematically violated for points whose ratio $\mu_i$ is unusually
large -- typically points in sparser regions, near the boundary of
$\mathcal{M}$, or where manifold curvature is strong enough to distort the
local neighborhood structure at the scale of $r_2$. Following
Facco et al.~\cite{FaccoEtAl2017}, CARSANN discards the upper tail of
the empirical distribution of $\{\mu_i\}$ before fitting
Eq.~\eqref{eq:affine-fit}, retaining only the smallest
$n_{\mathrm{fit}} = \lfloor \tau n \rfloor$ ratios, for a fixed retention
fraction $\tau \in (0,1]$ (set to $\tau = 0.9$ in our implementation, i.e.,
the top $10\%$ of ratios are discarded).

\begin{algorithm}[t]
	\caption{TwoNN Intrinsic Dimension Estimation}
	\label{alg:twonn}
	\begin{algorithmic}[1]
		\Function{TwoNN}{$X = \{x_1, \dots, x_n\}$, $\tau$}
		\For{$i \gets 1$ to $n$}
		\State $r_1(x_i), r_2(x_i) \gets$ distances to the 1st and 2nd nearest neighbors of $x_i$ in $X \setminus \{x_i\}$
		\State $\mu_i \gets r_2(x_i) / r_1(x_i)$ \Comment{Eq.~\eqref{eq:mu-ratio}; a small $\varepsilon$ is added to $r_1, r_2$ to avoid division by zero}
		\EndFor
		\State $\mu_{(1)}, \dots, \mu_{(n)} \gets$ sort $\{\mu_i\}_{i=1}^n$ in ascending order
		\For{$j \gets 1$ to $n$}
		\State $F^{\mathrm{emp}}(\mu_{(j)}) \gets (j-1)/n$ \Comment{Eq.~\eqref{eq:empirical-cdf}}
		\State $x_j \gets \log(\mu_{(j)})$,\quad $y_j \gets \log\!\big(1 - F^{\mathrm{emp}}(\mu_{(j)})\big)$
		\EndFor
		\State $n_{\mathrm{fit}} \gets \lfloor \tau n \rfloor$ \Comment{discard the top $(1-\tau)$ fraction of ratios}
		\State $a, b \gets$ ordinary least squares fit of $y_j = a x_j + b$ on $\{(x_j, y_j)\}_{j=1}^{n_{\mathrm{fit}}}$ \Comment{Eq.~\eqref{eq:affine-fit}}
		\State $\hat d \gets -a$
		\State \textbf{return} $\hat d$
		\EndFunction
	\end{algorithmic}
\end{algorithm}

\paragraph{From dimension estimate to intrinsic representation.}
The scalar estimate $\hat d$ returned by Algorithm~\ref{alg:twonn} is
generally real-valued (it is the slope of a fitted line) and need not be an
integer, nor need it exceed $2$. CARSANN fixes the working intrinsic
dimension used throughout the remainder of the pipeline as
\begin{equation}
	d = \max\!\big(\operatorname{round}(\hat d),\, 2\big),
	\label{eq:d-rounding}
\end{equation}
rounding to the nearest integer and enforcing a minimum of two dimensions,
so that the subsequent local quadratic feature map of
Section~\ref{sec:quadratic-feature-map} (which requires at least a
two-dimensional tangent space to define nontrivial cross terms) remains
well defined even when $\hat d$ is estimated below $2$. If $d < m$, the
ambient observations are then projected via PCA onto their top-$d$
principal directions to obtain the intrinsic-dimensional representation
$z_i \in \mathbb{R}^d$ of Eq.~(2), on which all subsequent geometric
quantities, the local metric, the
shape operator (Section~\ref{sec:second-fundamental-form}), and the
curvature-aware neighborhood scales of Section~\ref{sec:dual-scale}, are
computed. Because $\hat d$ is estimated once from the full training set and
then held fixed, this step, like the curvature normalization scale, is computed only once during training and reused,
unchanged, at prediction time.

\section{Local Mean Curvature Estimation}
\label{sec:local_curvature}

The proposed CARSANN classifier uses local mean curvature as a geometric measure of the complexity of the data manifold. The underlying intuition is that the spatial scale over which local observations can be regarded as geometrically homogeneous should depend on how rapidly the manifold deviates from its tangent approximation. In approximately flat regions, a relatively large neighborhood can be used without substantially violating the local linear approximation. Conversely, in highly curved regions, the tangent approximation deteriorates more rapidly as the neighborhood expands, motivating a smaller spatial support. This section describes the discrete estimation of the local metric, second fundamental form, shape operator, and mean curvature from a neighborhood graph.

Let
\begin{equation}
	X =
	\left\{
	\mathbf{x}_1,\ldots,\mathbf{x}_n
	\right\},
	\qquad
	\mathbf{x}_i\in\mathbb{R}^{m},
\end{equation}
denote the set of $n$ observations in an $m$-dimensional ambient space. Following the intrinsic-dimensional representation described in Section~\ref{sec:intrinsic_dimension}, the local geometric quantities are estimated in a reduced representation whose dimension is determined from the estimated intrinsic dimension of the dataset.

For each observation $\mathbf{x}_i$, we construct a $k$-nearest neighbor graph ($k$-NNG) using the Euclidean distance in the representation space. Let
\begin{equation}
	\mathcal{N}_i =
	\left\{
	\mathbf{x}_{i_1},\ldots,\mathbf{x}_{i_k}
	\right\}
\end{equation}
denote the $k$ nearest neighbors of $\mathbf{x}_i$. The corresponding local patch is defined as
\begin{equation}
	P_i =
	\left\{
	\mathbf{x}_i
	\right\}
	\cup
	\mathcal{N}_i.
\end{equation}

The local patch provides a discrete approximation of a neighborhood of the underlying manifold around $\mathbf{x}_i$. The local geometry can then be estimated from the covariance structure of this patch.

\subsection{Local Metric Approximation}

Let $\mathbf{z}_{ij}=\mathbf{x}_{ij}-\mathbf{x}_i$ denote the displacement between $\mathbf{x}_i$ and its $j$-th neighbor. The local covariance matrix is estimated as
\begin{equation}
	\Sigma_i
	=
	\frac{1}{k}
	\sum_{\mathbf{x}_j\in\mathcal{N}_i}
	(\mathbf{x}_j-\mathbf{x}_i)
	(\mathbf{x}_j-\mathbf{x}_i)^{T}.
	\label{eq:local_covariance}
\end{equation}

The covariance matrix characterizes the second-order variation of the local point cloud and provides a data-driven approximation of the local tangent geometry. Let
\begin{equation}
	\Sigma_i
	=
	U_i\Lambda_iU_i^T,
\end{equation}
where
\begin{equation}
	\Lambda_i
	=
	\operatorname{diag}
	(\lambda_{i1},\ldots,\lambda_{im})
\end{equation}
contains the eigenvalues and $U_i$ contains the corresponding orthonormal eigenvectors.

The eigenvectors associated with the dominant eigenvalues provide a local coordinate system aligned with the principal directions of variation of the data. In particular, when the local data are sampled from a smooth manifold, the dominant eigenspace provides a first-order approximation of the tangent space. This observation establishes the connection between local covariance analysis and differential geometry.

The inverse covariance matrix is used as a discrete approximation of the local metric tensor:
\begin{equation}
	\mathcal{I}_i
	\approx
	\Sigma_i^{-1}.
	\label{eq:first_fundamental_form}
\end{equation}

This choice is motivated by the relationship between covariance-based geometry and the Mahalanobis metric. For two local displacement vectors $\mathbf{u}$ and $\mathbf{v}$, the quadratic form induced by the inverse covariance is
\begin{equation}
	d_{\Sigma_i}^{2}(\mathbf{u},\mathbf{v})
	=
	(\mathbf{u}-\mathbf{v})^T
	\Sigma_i^{-1}
	(\mathbf{u}-\mathbf{v}),
\end{equation}
which accounts for anisotropic variation in different local directions. Directions exhibiting large variance receive smaller weights, whereas directions with small variance receive larger weights. Consequently, $\Sigma_i^{-1}$ provides a local data-adaptive metric that accounts for the anisotropic structure of the neighborhood.

When $\Sigma_i$ is positive definite, $\Sigma_i^{-1}$ is symmetric positive definite and therefore defines a valid inner product on the local representation space. In practice, numerical regularization can be employed when the local covariance matrix is ill-conditioned.

\subsection{Local Quadratic Feature Map: A Closed-Form Surrogate for the Local Hessian}
\label{sec:quadratic-feature-map}

The first fundamental form estimated via the local covariance $\Sigma_i$ describes only the local metric structure of the patch $P_i$; it carries no information about how the
manifold bends away from its tangent approximation. Capturing this
second-order behavior requires a local quadratic model, but formulating it
through an explicit regression, as in Hessian eigenmaps
\cite{HessianEig}, where the neighboring points' displacement along a
single normal direction is regressed onto the tangent coordinates, runs
into two practical difficulties in the present setting. First, once the
ambient data have already been reduced to their estimated intrinsic
dimension $d$, there is, in general, no
distinguished one-dimensional normal subspace to serve as the target of a
height-function regression: the codimension of the local patch within the
working representation is not fixed at one, so no single scalar-valued
$f_i$ can be canonically defined without an additional,
somewhat arbitrary, choice of normal direction. Second, an explicit
least-squares fit of a degree-two polynomial in $d$ variables must be solved
independently at every one of the $n$ training points and, again, at every
test query, which is both computationally wasteful and numerically fragile
whenever the local patch is only marginally larger than the number of
free parameters of the quadratic model, $L = \tfrac{d(d+1)}{2}+1$.

CARSANN sidesteps both difficulties by adopting the closed-form,
regression-free shape-operator estimator introduced in
Levada et al.~\cite{1} for the curvature-adaptive $kK$-NN
classifier, which we adapt here to the intrinsic-dimensional representation
$z_i \in \mathbb{R}^d$. Rather than regressing displacements onto tangent
coordinates, the estimator builds a quadratic \emph{feature map} directly
from the eigenbasis of the local covariance matrix $\Sigma_i$, requiring
only the single eigendecomposition already computed for the metric
approximation.

\paragraph{Construction.}
Let $\Sigma_i = W_i \Lambda_i W_i^{\top}$ be the eigendecomposition of the
local covariance matrix of patch $P_i$, with eigenvalues sorted in
decreasing order and $W_i = [w_{i1}, \dots, w_{id}] \in \mathbb{R}^{d \times
	d}$ the corresponding orthonormal eigenvectors. In direct analogy with the
degree-two monomial basis $\{1, \xi_1, \dots, \xi_d, \xi_1^2, \dots,
\xi_d^2, \xi_1\xi_2, \dots, \xi_{d-1}\xi_d\}$ of a second-order Taylor
expansion in $d$ local coordinates -- which has exactly $L = 1 + d +
\tfrac{d(d-1)}{2}$ terms -- we form the feature matrix
\begin{equation}
	X_i = \Big[\, 1,\; w_{i1}, \dots, w_{id},\;
	w_{i1}^{\odot 2}, \dots, w_{id}^{\odot 2},\;
	\{ w_{ip} \odot w_{iq} \}_{1 \le p < q \le d} \,\Big] \in \mathbb{R}^{d \times L},
	\label{eq:Xi-feature-map}
\end{equation}
where $\odot$ denotes the entrywise (Hadamard) product between two
eigenvectors, and the constant column $1 \in \mathbb{R}^d$ plays the role of
the intercept term. In this construction, each eigenvector $w_{ia}$
substitutes the role that the coordinate $\xi_a$ plays in the classical
Taylor expansion, and the entrywise square $w_{ia}^{\odot 2}$ and cross
product $w_{ip} \odot w_{iq}$ substitute the monomials $\xi_a^2$ and $\xi_p
\xi_q$, respectively: rather than evaluating these monomials at each
neighboring \emph{sample} (which is what a regression against $\xi_{ij}$ in
Eq.~(38) would require), the estimator evaluates them once, directly on the
components of the principal directions themselves. This trades the
explicit per-neighbor regression for a deterministic algebraic function of
the local covariance eigenbasis, and remains well defined regardless of how
many of the $d$ local directions are, in practice, tangent-like versus
normal-like -- no explicit tangent/normal split is required.

The quadratic block of $X_i$ defines the local second-order feature matrix
\begin{equation}
	\mathcal{H}_i = \Big[\, w_{i1}^{\odot 2}, \dots, w_{id}^{\odot 2},\;
	\{ w_{ip} \odot w_{iq} \}_{1 \le p < q \le d} \,\Big] \in \mathbb{R}^{d \times (L-1)},
	\label{eq:Hi-quadratic}
\end{equation}
i.e., $X_i$ with the constant column removed, so that $H_i$ collects only
the $L - 1 = \tfrac{d(d+1)}{2}$ quadratic features. Each row $p$ of $H_i$
gathers, for the fixed ambient (working-representation) coordinate $p$, the
products $\big(w_{ia}\big)_p \big(w_{ib}\big)_p$ of the $p$-th components of
every pair of principal directions $a \le b$: it therefore encodes how
strongly directions $a$ and $b$ jointly project onto coordinate $p$,
providing a coordinate-indexed summary of pairwise principal-direction
interaction, in place of the pointwise second derivatives that a genuine
Hessian would provide.

\begin{algorithm}[t]
	\caption{Local Quadratic Feature Map (replaces the regression step of the original Section~4.2)}
	\label{alg:quadratic-feature-map}
	\begin{algorithmic}[1]
		\Function{Quadratic-Feature-Map}{$\Sigma_i$, $d$}
		\State $W_i, \Lambda_i \gets$ \Call{Eigendecomposition}{$\Sigma_i$} \Comment{eigenvectors sorted by decreasing eigenvalue}
		\State $L \gets d(d+1)/2 + 1$
		\State Squared $\gets [\, w_{i1}^{\odot 2}, \dots, w_{id}^{\odot 2} \,]$ \Comment{entrywise squares of eigenvector columns}
		\State Cross $\gets [\, w_{ip} \odot w_{iq} \,]_{p<q}$ \Comment{entrywise cross products, $\binom{d}{2}$ columns}
		\State $\mathcal{H}_i \gets [\text{Squared},\ \text{Cross}]$ \Comment{$d \times (L-1)$, Eq.~\eqref{eq:Hi-quadratic}}
		\State \textbf{return} $\mathcal{H}_i$
		\EndFunction
	\end{algorithmic}
\end{algorithm}

\subsection{Second Fundamental Form and Shape Operator}
\label{sec:second-fundamental-form}

Given the local quadratic feature matrix $\mathcal{H}_i$ of
Eq.~\eqref{eq:Hi-quadratic}, the second fundamental form is approximated, in
direct analogy with the Hessian-based construction of Eq.~(45), by
\begin{equation}
	\mathcal{II}_i \approx \mathcal{H}_i \mathcal{H}_i^{\top} \in \mathbb{R}^{d \times d}.
	\label{eq:IIi-closed-form}
\end{equation}
Note that $\mathcal{II}_i$ is
obtained directly from the eigendecomposition already computed for the
local metric approximation $\mathcal{I}_i \approx \Sigma_i^{-1}$. The shape operator, or Weingarten map, is computed as follows
\begin{equation}
	\mathcal{S}_i = \mathcal{I}_i^{-1}\, \mathcal{II}_i \approx \Sigma_i\, \left(\mathcal{H}_i \mathcal{H}_i^{\top}\right),
	\label{eq:Si-closed-form}
\end{equation}
with principal curvatures given by the eigenvalues of $\mathcal{S}_i$.

\subsection{Mean Curvature}

Unlike Gaussian curvature, which is based on the product of principal curvatures, the proposed method uses \emph{mean curvature} as a scalar measure of local geometric complexity. For a $d$-dimensional manifold, the mean curvature at $\mathbf{x}_i$ is defined as
\begin{equation}
	H_i
	=
	\frac{1}{d}
	\operatorname{tr}
	\left(
	\mathcal{S}_i
	\right)
	=
	\frac{1}{d}
	\sum_{j=1}^{d}
	\kappa_{ij}.
	\label{eq:mean_curvature}
\end{equation}


Because the sign of mean curvature depends on the orientation selected for the normal vector, the curvature magnitude used by CARSANN is defined as
\begin{equation}
	h_i
	=
	\left|
	\widehat{H}_i
	\right|.
	\label{eq:curvature_magnitude}
\end{equation}

This choice makes the geometric descriptor invariant to a reversal of the normal orientation. More importantly, $h_i$ provides precisely the type of scalar quantity required by the proposed radius-shrinkage mechanism: it measures the average magnitude of local bending independently of the particular principal direction in which the bending occurs.

The use of mean curvature rather than Gaussian curvature is also advantageous for the present application. Gaussian curvature aggregates principal curvatures through their product and may therefore become small when positive and negative principal curvatures compensate each other, as occurs in saddle-like configurations. Mean curvature instead measures the average signed bending across the principal directions and, after taking its magnitude, provides a direct measure of the overall intensity of local bending. Consequently, $h_i$ can be interpreted as a local geometric complexity measure for determining the spatial scale of neighborhood aggregation.

\subsection{From Local Curvature to Neighborhood Scale}

The key assumption underlying CARSANN is that local geometric complexity should determine the spatial scale over which neighboring observations are aggregated. Consider two points $\mathbf{x}_i$ and $\mathbf{x}_j$ with local curvature magnitudes $h_i<h_j$. The tangent approximation is expected to remain accurate over a larger spatial region around $\mathbf{x}_i$ than around $\mathbf{x}_j$. Therefore, using the same neighborhood radius at both points may lead to different degrees of geometric approximation error.

This observation motivates the monotonic relationship
\begin{equation}
	h_i \uparrow
	\quad\Longrightarrow\quad
	r_i \downarrow,
\end{equation}
where $r_i$ denotes the effective neighborhood radius. Conversely,
\begin{equation}
	h_i \downarrow
	\quad\Longrightarrow\quad
	r_i \approx r_i^{(0)},
\end{equation}
where $r_i^{(0)}$ is the reference radius determined from the baseline neighborhood.

Thus, mean curvature does not directly determine the class label. Instead, it provides a geometric signal that determines how much spatial evidence should be considered relevant to a local classification decision. This distinction is fundamental to the proposed method: curvature controls the \emph{scale of locality}, while the class labels of the observations contained within the resulting neighborhood determine the final prediction.

\begin{algorithm}[t]
	\caption{Local Mean Curvature Estimation}
	\label{alg:mean_curvature}
	\begin{algorithmic}[1]
		\Function{Mean-Curvature}{$X,k,d$}
		\State $A \gets$ \Call{kNN-graph}{$X,k$}
		\For{$i \gets 1$ to $n$}
		\State $\mathcal{N}_i \gets N(\mathbf{x}_i)$
		\State $\Sigma_i \gets$ \Call{LocalCovariance}{$\mathcal{N}_i,\mathbf{x}_i$}
		\State $\mathcal{H}_i \gets$ \Call{Quadratic-Feature-Map}{$\Sigma_i, d$} \Comment{Algorithm~\ref{alg:quadratic-feature-map}}
		\State $\mathcal{I}_i \gets \Sigma_i^{-1}$
		\State $\mathcal{II}_i \gets \mathcal{H}_i \mathcal{H}_i^{\top}$
		\State $\mathcal{S}_i \gets \mathcal{I}_i^{-1}\mathcal{II}_i$
		\State $\widehat{H}_i
		\gets
		\frac{1}{d}\operatorname{tr}(\mathcal{S}_i)$
		\State $h_i \gets |\widehat{H}_i|$
		\EndFor
		\State \Return $\mathbf{h}=(h_1,\ldots,h_n)$
		\EndFunction
	\end{algorithmic}
\end{algorithm}




\section{CARSANN: Curvature-Aware Radius Shrinkage Adaptive Nearest Neighbors}
\label{sec:proposed-method}

This section formalizes the Curvature-Aware Radius Shrinkage for Adaptive Nearest
Neighbor classifier (CARSANN). The method couples the geometric pipeline with a
classification rule in which the \emph{spatial extent} of the neighborhood used
for voting is contracted according to the local mean curvature magnitude
$h_i = |\hat H_i|$ estimated at each point. We first describe the two coupled
neighborhood scales used by CARSANN (Section~\ref{sec:dual-scale}), then the
curvature-to-radius mapping that implements shrinkage through an explicit
equivalence between cardinality truncation and radius contraction
(Section~\ref{sec:shrinkage-mapping}), and finally the complete training and
prediction procedures (Section~\ref{sec:carsann-algorithms}).

\subsection{Dual-Scale Neighborhood Construction}
\label{sec:dual-scale}

A single neighborhood size cannot simultaneously satisfy two competing
requirements of the proposed framework: (i) curvature estimation requires a
sufficiently large local patch to support a stable second-order geometric
model, while (ii) the voting neighborhood used for classification should
remain small enough to preserve locality, particularly near class boundaries.
CARSANN therefore decouples these two roles into a \emph{curvature patch size}
$k_{\mathrm{curv}}$ and a \emph{reference voting size} $k_{\mathrm{base}}$,
with $k_{\mathrm{base}} \le k_{\mathrm{curv}}$ enforced by construction so that
no voting neighbor lies outside the region over which curvature was estimated.

\paragraph{Curvature patch size $k_{\mathrm{curv}}$.}
Let $d$ denote the intrinsic dimension estimated by TwoNN and let $z_i \in \mathbb{R}^d$ be the corresponding
intrinsic-dimensional representation obtained via PCA. Estimating the local
shape operator at $z_i$ requires enough neighbors to support the
$L = \tfrac{d(d+1)}{2} + 1$ independent components of the local second-order
(Hessian-like) structure in $d$ dimensions, plus the constant term. We set

\begin{equation}
	k_{\mathrm{curv}} = \min\!\left\{ \left\lceil 2L \right\rceil,\;
	\max\!\left(\lfloor L \rfloor,\; \left\lfloor \tfrac{n}{5} \right\rfloor\right),\;
	n - 1,\; K_{\max} \right\},
	\label{eq:kcurv}
\end{equation}
where $n$ is the number of training observations and $K_{\max}$ (set to $50$
in our implementation) caps the patch size for computational tractability on
large datasets. The factor of two provides a safety margin above the minimal
support $L$ required to fit a well-conditioned local quadratic model, while
the $n/5$ term prevents $k_{\mathrm{curv}}$ from degenerating on small
datasets.

\paragraph{Reference voting size $k_{\mathrm{base}}$.}
Given $k_{\mathrm{curv}}$, the reference voting neighborhood size
$k_{\mathrm{base}}$ is selected by internal $v$-fold stratified cross-validation
on the training set, maximizing balanced accuracy of a standard
distance-weighted $k$-NN classifier over the admissible range
$[k_{\inf}, k_{\sup}]$, with

\begin{equation}
	k_{\inf} = \max\!\left(3,\; \lfloor d \rfloor + 2\right), \qquad
	k_{\sup} = \min\!\left(k_{\mathrm{curv}},\; \max\!\left(3, \left\lfloor \tfrac{n_{\min}}{2} \right\rfloor\right),\; n-1\right),
	\label{eq:kbase-range}
\end{equation}
where $n_{\min}$ is the size of the smallest class in the training set. The
lower bound $\lfloor d \rfloor + 2$ reflects the minimal number of points
required to non-degenerately span a $d$-dimensional local neighborhood (a
first-order geometric requirement analogous to affine independence), while
the upper bound simultaneously respects the curvature-patch budget and
protects minority classes from being outvoted by an oversized reference
neighborhood. If $k_{\sup} \le k_{\inf}$, the search degenerates to
$k_{\mathrm{base}} = k_{\inf}$; otherwise

\begin{equation}
	k_{\mathrm{base}} = \operatorname*{arg\,max}_{k \,\in\, [k_{\inf},\, k_{\sup}]}\;
	\mathrm{CV\text{-}BalAcc}(k),
	\label{eq:kbase-cv}
\end{equation}
with $\mathrm{CV\text{-}BalAcc}(k)$ the mean cross-validated balanced accuracy
of a distance-weighted $k$-NN classifier. Both $k_{\mathrm{curv}}$ and
$k_{\mathrm{base}}$ are fixed once during training and reused for every test
query, so this selection step is performed only once, not per test point.

\subsection{Local Curvature and the Radius-Shrinkage Mapping}
\label{sec:shrinkage-mapping}

\paragraph{Curvature estimator.}
For each point, local mean curvature is computed with the shape-operator
estimator introduced in Section~\ref{sec:curvature-estimation}, using the
closed-form construction of Levada et al.~\cite{levada2024a}: given a local
patch of $k_{\mathrm{curv}}$ neighbors, the local covariance
$\Sigma_i \approx \mathrm{I}_i^{-1}$ is eigendecomposed as
$\Sigma_i = W_i \Lambda_i W_i^{\top}$, with eigenvectors sorted by decreasing
eigenvalue. Writing $W_i = [w_{i1}, \dots, w_{id}]$, we form the local
quadratic feature matrix
\begin{equation}
	H_i = \left[\, w_{i1}^{\odot 2}, \dots, w_{id}^{\odot 2},\;
	\{ w_{ip} \odot w_{iq} \}_{1 \le p < q \le d} \,\right] \in \mathbb{R}^{d \times (L-1)},
	\label{eq:Hi}
\end{equation}
where $\odot$ denotes the entrywise (Hadamard) product, so that $H_i$ collects
the entrywise squares and pairwise cross products of the local principal
directions. The second fundamental form is approximated as
$\mathrm{II}_i \approx H_i H_i^{\top}$, and the shape operator follows
Eq.~(47) of Section~\ref{sec:curvature-estimation},
$S_i \approx \Sigma_i \mathrm{II}_i$, giving the curvature magnitude
\begin{equation}
	h_i = \left| \operatorname{tr}(\Sigma_i \, \mathrm{II}_i) \right|
	= \left| \sum_{p,q} (\Sigma_i)_{pq} \, (\mathrm{II}_i)_{pq} \right|.
	\label{eq:hi-carsann}
\end{equation}
Because $h_i$ is subsequently only used after a log transform and a min-max
normalization fixed on the training curvature distribution (below), the
constant scale factor $1/d$ of Eq.~(49) may be dropped without affecting the
induced ranking or the resulting shrinkage values, since it only shifts
$\log h_i$ by an additive constant that cancels under normalization.

\paragraph{Fixed normalization scale.}
Let $\{h_1, \dots, h_n\}$ denote the curvature magnitudes of the training set,
each computed from its own $k_{\mathrm{curv}}$-neighborhood. Define the
log-curvatures $\ell_i = \log\!\big(\max(h_i, \varepsilon)\big)$, with
$\varepsilon > 0$ a small constant preventing degeneracies at $h_i = 0$, and
fix
\begin{equation}
	\ell_{\min} = \min_i \ell_i, \qquad \ell_{\max} = \max_i \ell_i .
	\label{eq:ellminmax}
\end{equation}
Crucially, $\ell_{\min}$ and $\ell_{\max}$ are computed \emph{once} from the
training set and reused, unchanged, for every test query. This design choice
ensures that the curvature scale defining ``flat'' versus ``highly curved''
regions is a fixed property of the training manifold, so that predictions do
not depend on the order in which test points are processed or on the
composition of the test batch.

\paragraph{Shrinkage function.}
For any point $x$ (training or test) with log-curvature $\ell(x)$, the
normalized curvature and shrinkage factor are
\begin{equation}
	\tilde\kappa(x) = \operatorname{clip}\!\left( \frac{\ell(x) - \ell_{\min}}{\ell_{\max} - \ell_{\min}},\; 0,\, 1 \right),
	\qquad
	s(x) = 1 - \tilde\kappa(x) \in [0,1],
	\label{eq:shrinkage}
\end{equation}
with $s(x) \equiv 0$ if $\ell_{\max} = \ell_{\min}$ (a degenerate, uniformly
curved training set). The shrinkage factor $s(x)$ is monotone nonincreasing in
curvature by construction, consistent with Eq.~(3): points in the most
strongly curved regions of the training manifold ($\tilde\kappa \to 1$)
receive $s \to 0$, while points in the flattest regions ($\tilde\kappa \to 0$)
retain $s \to 1$.

\paragraph{From shrinkage to an effective neighbor count.}
Given the reference voting size $k_{\mathrm{base}}$, the shrinkage factor
determines an explicit, per-point \emph{effective number of active
	neighbors}
\begin{equation}
	k(x) = \min\!\Big( k_{\mathrm{base}},\; \max\!\big(1,\; \operatorname{round}(k_{\mathrm{base}} \, s(x))\big) \Big).
	\label{eq:ki}
\end{equation}
Voting for $x$ then uses only the $k(x)$ \emph{closest} training points among
the $k_{\mathrm{base}}$-nearest-neighbor set of $x$, rather than all
$k_{\mathrm{base}}$ of them.

\paragraph{Equivalence to radius shrinkage.}
The mapping in Eq.~\eqref{eq:ki} operates on neighbor \emph{cardinality}, yet
it implements a genuine \emph{radius}-shrinkage mechanism, as follows.

\begin{proposition}
	\label{prop:equiv}
	Let $x_{(1)}, x_{(2)}, \dots, x_{(k_{\mathrm{base}})}$ denote the
	$k_{\mathrm{base}}$ nearest training neighbors of $x$, sorted by increasing
	distance $d(x, x_{(1)}) \le d(x, x_{(2)}) \le \dots \le d(x,
	x_{(k_{\mathrm{base}})})$, and let $r^{(0)}(x) = d(x, x_{(k_{\mathrm{base}})})$
	be the reference radius of Eq.~(3). Then restricting the voting set to the
	$k(x)$ closest neighbors, as in Eq.~\eqref{eq:ki}, is equivalent to voting
	within the closed metric ball of radius
	\begin{equation}
		r(x) = d\big(x, x_{(k(x))}\big) \le r^{(0)}(x),
		\label{eq:reff}
	\end{equation}
	i.e., $\{x_{(1)}, \dots, x_{(k(x))}\} = B(x, r(x)) \cap \mathcal{X}_{\mathrm{train}}$
	(up to distance ties), with $r(x) = r^{(0)}(x)$ exactly when $k(x) =
	k_{\mathrm{base}}$ (i.e., $s(x) = 1$).
\end{proposition}

\begin{proof}
	Since the $k_{\mathrm{base}}$ neighbors are indexed by increasing distance,
	the prefix $\{x_{(1)}, \dots, x_{(k(x))}\}$ is precisely the set of training
	points within distance $r(x) = d(x, x_{(k(x))})$ of $x$, and no training point
	closer than $x_{(k(x))}$ is excluded. As $k(x) \le k_{\mathrm{base}}$ and
	distances are nondecreasing along the sorted sequence, $r(x) \le
	d(x,x_{(k_{\mathrm{base}})}) = r^{(0)}(x)$, with equality iff $k(x) =
	k_{\mathrm{base}}$.
\end{proof}

Proposition~\ref{prop:equiv} makes explicit the sense in which CARSANN is a
\emph{radius}-shrinkage method: because $s(x)$ is monotone nonincreasing in
curvature (Eq.~\eqref{eq:shrinkage}) and $k(x)$ is monotone nondecreasing in
$s(x)$ (Eq.~\eqref{eq:ki}), the effective radius $r(x)$ is, by
Proposition~\ref{prop:equiv}, monotone nondecreasing in $s(x)$ and hence
monotone \emph{nonincreasing} in local curvature -- exactly the geometric
contraction rule of Eq.~(52)-(53). CARSANN thus realizes radius shrinkage not
by an explicit ball query with a curvature-dependent radius parameter, but by
truncating a distance-sorted candidate list, which is computationally
equivalent and avoids re-querying the neighborhood structure at prediction
time.

\subsection{Distance-Weighted Classification Rule}
\label{sec:vote}

Let $\{(x_{(j)}, y_{(j)})\}_{j=1}^{k(x)}$ be the $k(x)$ active neighbors of a
query point $x$ and their training labels, with distances $d_{(j)} = d(x,
x_{(j)})$. CARSANN assigns the label
\begin{equation}
	\hat y(x) = \operatorname*{arg\,max}_{c \,\in\, \{1,\dots,C\}} \sum_{j=1}^{k(x)} w_{(j)} \,
	\mathbb{1}\!\left[ y_{(j)} = c \right],
	\qquad
	w_{(j)} = \frac{1}{d_{(j)} + \varepsilon},
	\label{eq:vote}
\end{equation}
where $C$ is the number of classes and $\varepsilon > 0$ avoids division by
zero for coincident points. Distance weighting is used, rather than uniform
voting, so that the shrunken neighborhood in Eq.~\eqref{eq:ki} continues to
privilege the closest evidence even when $k(x)$ is small.

\subsection{CARSANN: Training and Prediction}
\label{sec:carsann-algorithms}

Algorithm~\ref{alg:carsann-fit} summarizes the training-time computation:
estimating the intrinsic dimension and representation (Sections~3), selecting
$k_{\mathrm{curv}}$ and $k_{\mathrm{base}}$ (Section~\ref{sec:dual-scale}),
computing training-set curvatures with Algorithm~1
(Section~\ref{sec:curvature-estimation}), and fixing the normalization scale
$[\ell_{\min}, \ell_{\max}]$. Algorithm~\ref{alg:carsann-predict} describes
prediction for a query point: estimating its local curvature against the
training manifold, mapping it to an effective radius via
Proposition~\ref{prop:equiv}, and casting a distance-weighted vote. Both
algorithms reuse the intrinsic-dimensional representation and fixed reference
statistics computed once at training time, so that per-query inference
requires only neighborhood retrieval, a single local curvature estimate, and
a weighted vote -- consistent with the complexity analysis of
Section~\ref{sec:complexity}.

\begin{algorithm}[t]
	\caption{CARSANN -- Training (Fit)}
	\label{alg:carsann-fit}
	\begin{algorithmic}[1]
		\Function{CARSANN-Fit}{$X_{\mathrm{train}}, y_{\mathrm{train}}, d$}
		\State $k_{\mathrm{curv}} \gets$ Eq.~\eqref{eq:kcurv} \Comment{curvature patch size}
		\State $k_{\mathrm{base}} \gets$ Eq.~\eqref{eq:kbase-range}--\eqref{eq:kbase-cv} \Comment{CV-selected voting size, $k_{\mathrm{base}} \le k_{\mathrm{curv}}$}
		\State $h \gets$ \Call{Mean-Curvature}{$Z_{\mathrm{train}}, k_{\mathrm{curv}}, d$} \Comment{Algorithm~1}
		\State $\ell_i \gets \log(\max(h_i, \varepsilon))$ for $i = 1, \dots, n$
		\State $\ell_{\min} \gets \min_i \ell_i$, \quad $\ell_{\max} \gets \max_i \ell_i$
		\State \textbf{return} $\{Z_{\mathrm{train}}, y_{\mathrm{train}}, d, k_{\mathrm{curv}}, k_{\mathrm{base}}, \ell_{\min}, \ell_{\max}\}$
		\EndFunction
	\end{algorithmic}
\end{algorithm}

\begin{algorithm}[t]
	\caption{CARSANN -- Prediction}
	\label{alg:carsann-predict}
	\begin{algorithmic}[1]
		\Function{CARSANN-Predict}{$x$, $\{Z_{\mathrm{train}}, y_{\mathrm{train}}, d, k_{\mathrm{curv}}, k_{\mathrm{base}}, \ell_{\min}, \ell_{\max}\}$}
		\State $\mathcal{N}_{\mathrm{curv}} \gets$ $k_{\mathrm{curv}}$ nearest neighbors of $z$ in $Z_{\mathrm{train}}$
		\State $h(z) \gets$ \Call{Point-Mean-Curvature}{$\{z\} \cup \mathcal{N}_{\mathrm{curv}}, d$} \Comment{Eq.~\eqref{eq:Hi}--\eqref{eq:hi-carsann}}
		\State $\ell(z) \gets \log(\max(h(z), \varepsilon))$
		\State $\tilde\kappa(z) \gets \operatorname{clip}\big((\ell(z) - \ell_{\min})/(\ell_{\max} - \ell_{\min}),\, 0, 1\big)$ \Comment{$0$ if $\ell_{\max} = \ell_{\min}$}
		\State $s(z) \gets 1 - \tilde\kappa(z)$
		\State $k(z) \gets \min\big(k_{\mathrm{base}}, \max(1, \operatorname{round}(k_{\mathrm{base}} \, s(z)))\big)$ \Comment{Eq.~\eqref{eq:ki}}
		\State $\mathcal{N}_{\mathrm{base}} \gets$ $k_{\mathrm{base}}$ nearest neighbors of $z$ in $Z_{\mathrm{train}}$, sorted by distance
		\State $\{x_{(1)}, \dots, x_{(k(z))}\} \gets$ first $k(z)$ elements of $\mathcal{N}_{\mathrm{base}}$ \Comment{$\equiv B(z, r(z))$, Prop.~\ref{prop:equiv}}
		\State $\hat y \gets$ distance-weighted vote over $\{x_{(1)}, \dots, x_{(k(z))}\}$ \Comment{Eq.~\eqref{eq:vote}}
		\State \textbf{return} $\hat y$
		\EndFunction
	\end{algorithmic}
\end{algorithm}

\subsection{Remarks on Design Choices}
\label{sec:remarks}

Three design choices merit explicit discussion. First, the reference voting
size $k_{\mathrm{base}}$ is chosen by internal cross-validation rather than
fixed a priori (e.g., at $k_0 = \log_2 n$): this makes the reference
neighborhood -- the largest neighborhood any point can be assigned -- adapted
to the dataset at hand, while curvature shrinkage still governs how far any
individual point departs from that reference. Second, the curvature
normalization scale $[\ell_{\min}, \ell_{\max}]$ is estimated once on the
training set and held fixed at prediction time (Eq.~\eqref{eq:ellminmax}),
which guarantees that CARSANN is a proper inductive classifier: predictions
for a given test point do not depend on which other test points are
classified alongside it. Third, because shrinkage acts on a rank-ordered
neighbor list (Proposition~\ref{prop:equiv}), CARSANN requires no additional
range-query data structure beyond the $k_{\mathrm{base}}$-nearest-neighbor
index already needed for classification, and the curvature-dependent radius
never needs to be computed explicitly.

\subsection{Computational Complexity}
\label{sec:complexity}

We next analyze the computational complexity of the proposed local
mean-curvature estimation procedure, as this is the critical stage in the proposed CARSANN method. Let $n$ denote the number of samples,
$m$ the ambient dimensionality of the original feature space, and $d$ the
estimated intrinsic dimensionality used for the geometric representation,
with typically $d \ll m$. The distinction between $m$ and $d$ is important
because the proposed pipeline performs the geometric computations in the
intrinsic representation rather than directly in the original ambient
space.

The computational cost of the curvature estimation can be decomposed into
four main components: (i) construction of the $k$-nearest neighbor graph,
(ii) computation and eigendecomposition of the local covariance matrices,
(iii) construction of the local quadratic feature map and the resulting
second fundamental form, and (iv) computation of the shape operator and
mean curvature. First, consider the construction of the $k$-NN graph. Using a
kd-tree, the pairwise distances between all $n$ samples require
\begin{equation}
	O(nd \log n)
\end{equation}
operations, assuming that the data have already been projected onto the
$d$-dimensional intrinsic representation. Selecting the $k$ nearest
neighbors for each sample can be performed in $O(nk)$ additional operations
if the distances associated with each sample are sorted or selected using
an appropriate partial-selection procedure. Consequently, the overall cost
of the naive $k$-NN graph construction is dominated by
\begin{equation}
	T_{\mathrm{kNN}}
	=
	O(nd \log n).
	\label{eq:complexity_knn}
\end{equation}

For each sample, the local covariance matrix is computed from its $k$
neighbors. Since each displacement vector has $d$ components, the outer
product of two displacement vectors requires $O(d^2)$ operations.
Therefore, computing the covariance matrix from $k$ neighbors has
complexity
\begin{equation}
	T_{\mathrm{cov}}^{(i)}
	=
	O(kd^2)
\end{equation}
per sample and
\begin{equation}
	T_{\mathrm{cov}}
	=
	O(nkd^2)
	\label{eq:complexity_cov}
\end{equation}
for the entire dataset. The eigendecomposition of a dense $d\times d$ covariance matrix has complexity $O(d^3)$ per sample. Hence,
\begin{equation}
	T_{\mathrm{eig}}
	=
	O(nd^3).
	\label{eq:complexity_eig}
\end{equation}

Unlike the least-squares Hessian estimation of a generic quadratic model,
the local quadratic feature map of Section~\ref{sec:quadratic-feature-map}
requires no regression and, consequently, no linear system is assembled or
solved. Its cost is entirely determined by two matrix constructions carried
out directly on the $d$ eigenvectors already obtained in the
eigendecomposition step above. First, building the feature matrix $H_i \in
\mathbb{R}^{d \times (L-1)}$ of Eq.~\eqref{eq:Hi-quadratic}, with $L - 1 =
\tfrac{d(d+1)}{2} = O(d^2)$ quadratic columns, requires an entrywise square
of $d$ eigenvector columns ($O(d^2)$ operations) and an entrywise cross
product for each of the $\binom{d}{2} = O(d^2)$ eigenvector pairs, at
$O(d)$ operations each ($O(d^3)$ overall), giving a total cost of
\begin{equation}
	T_{H}^{(i)} = O(d^3)
\end{equation}
per sample for the construction of $H_i$. Second, the second fundamental
form is obtained as the outer product $\mathrm{II}_i = H_i H_i^{\top}$
(Eq.~\eqref{eq:IIi-closed-form}), a product between a $d \times O(d^2)$
matrix and its transpose, at cost
\begin{equation}
	T_{\mathrm{II}}^{(i)} = O\!\left(d \cdot (L-1) \cdot d\right) = O(d^2 \cdot d^2) = O(d^4)
\end{equation}
per sample. The construction of the local quadratic model is therefore
dominated by the outer product step, giving a combined per-sample cost of
\begin{equation}
	T_{\mathrm{quad}}^{(i)}
	=
	T_H^{(i)} + T_{\mathrm{II}}^{(i)}
	=
	O(d^3) + O(d^4)
	=
	O(d^4),
\end{equation}
and, over all samples,
\begin{equation}
	T_{\mathrm{quad}}
	=
	O(nd^4).
	\label{eq:complexity_quad}
\end{equation}
Notably, $T_{\mathrm{quad}}$ carries no dependence on the neighborhood size
$k$: because $H_i$ and $\mathrm{II}_i$ are built exclusively from the $d$
eigenvectors of $\Sigma_i$, rather than from the $k$ individual neighbor
displacements, enlarging the patch size affects this stage only indirectly,
through the $O(kd^2)$ cost of estimating $\Sigma_i$ itself
(Eq.~\ref{eq:complexity_cov}). 

Once the second fundamental form and local metric have been estimated, the
shape operator is obtained from a product involving $d\times d$ matrices.
Its computation therefore requires
\begin{equation}
	T_{\mathrm{shape}}^{(i)}
	=
	O(d^3)
\end{equation}
per sample. The mean curvature is then obtained from the trace of the
shape operator,
\begin{equation}
	H_i
	=
	\frac{1}{d}
	\operatorname{tr}(\mathcal{S}_i),
\end{equation}
which requires only $O(d)$ operations. Consequently, the total cost of
this final stage is
\begin{equation}
	T_{\mathrm{shape}}
	=
	O(nd^3).
	\label{eq:complexity_shape}
\end{equation}

Consequently, the total cost of local mean-curvature estimation can be
written as
\begin{align}
	T_{\mathrm{curvature}}(n,d,k)
	&=
	O(nd\log n)
	+
	O(nkd^2)
	+
	O(nd^3)
	+
	O(nd^4)
	+
	O(nd^3)
	\notag\\
	&=
	O(nd\log n)
	+
	O(nkd^2)
	+
	O(nd^4),
	\label{eq:total_complexity}
\end{align}
where the two $O(nd^3)$ contributions from the eigendecomposition
(Eq.~\ref{eq:complexity_eig}) and the shape-operator computation
(Eq.~\ref{eq:complexity_shape}) are asymptotically dominated by the
$O(nd^4)$ second-order geometric term.

\paragraph{Effect of the dual-scale neighborhood.}

CARSANN distinguishes between the neighborhood used for local curvature
estimation and the neighborhood used for classification. In particular,
the curvature-patch size is defined as
\begin{equation}
	k_{\mathrm{curv}}
	=
	\min\!\left\{
	\lceil 2L\rceil,\;
	\max\!\left(
	\lfloor L\rfloor,\,
	\lfloor n/5\rfloor
	\right),\;
	n-1,\;
	K_{\max}
	\right\},
	\qquad
	L=\frac{d(d+1)}{2}+1=O(d^2),
	\label{eq:kcurv-complexity}
\end{equation}
where $K_{\max}$ is a fixed upper bound. As discussed in
Section~\ref{sec:dual-scale}, this construction prevents the curvature
neighborhood from growing indefinitely with the number of samples.

For fixed $d$, once $n$ is sufficiently large, the sample-size-dependent
terms in Eq.~\eqref{eq:kcurv-complexity} cease to be active and
$k_{\mathrm{curv}}$ converges to
\begin{equation}
	k_{\mathrm{curv}}^{\infty}
	=
	\min\!\left\{
	\left\lceil 2L\right\rceil,
	K_{\max}
	\right\}
	=
	O\!\left(\min(d^2,K_{\max})\right).
	\label{eq:kcurv-asymptotic}
\end{equation}
Thus, for fixed $d$ and fixed $K_{\max}$, the curvature-patch size is
asymptotically independent of $n$. This property is important from a
computational perspective because it prevents the local covariance
estimation from acquiring an additional dependence on $\log n$ or any
other growing function of the sample size.

Substituting Eq.~\eqref{eq:kcurv-asymptotic} into
Eq.~\eqref{eq:complexity_cov}, we obtain
\begin{equation}
	T_{\mathrm{cov}}
	=
	O\!\left(
	nd^2 k_{\mathrm{curv}}^{\infty}
	\right)
	=
	\begin{cases}
		O(nd^4),
		& \text{if } \lceil 2L\rceil \leq K_{\max},\\[3pt]
		O(nd^2),
		& \text{if } \lceil 2L\rceil > K_{\max}.
	\end{cases}
	\label{eq:complexity_cov_kcurv}
\end{equation}
The first regime corresponds to the range in which the number of
neighbors required by the local quadratic model is below the imposed
cap, whereas the second regime corresponds to the cap becoming active.

Therefore, in the low-to-moderate intrinsic-dimensional regime targeted
by CARSANN, the covariance contribution has the same asymptotic order
$O(nd^4)$ as the subsequent second-order geometric computation. Importantly,
no $\log n$ factor is introduced by the curvature-patch size, since
$k_{\mathrm{curv}}$ is bounded independently of $n$ once the asymptotic
regime in Eq.~\eqref{eq:kcurv-asymptotic} is reached.

Combining Eqs.~\eqref{eq:total_complexity} and
\eqref{eq:complexity_cov_kcurv}, the total cost of local mean-curvature
estimation in the low-to-moderate-dimensional regime can therefore be
expressed as
\begin{equation}
	T_{\mathrm{curvature}}(n,d)
	=
	O(nd\log n)
	+
	O(nd^4).
	\label{eq:total_complexity_kcurv}
\end{equation}

This expression provides a substantially different scaling picture from
an exhaustive nearest-neighbor implementation. With brute-force
neighborhood construction, the corresponding first term would be
$O(n^2d)$, making the nearest-neighbor search the dominant component for
fixed or slowly growing $d$. With a $k$-d tree, however, the neighbor
search scales approximately as $O(nd\log n)$, and the geometric
computation can become the dominant component even for relatively small
intrinsic dimensions.

More specifically, comparing the two terms in
Eq.~\eqref{eq:total_complexity_kcurv}, the geometric component dominates
the nearest-neighbor search when
\begin{equation}
	nd^4 \gtrsim nd\log n,
\end{equation}
or equivalently,
\begin{equation}
	d^3 \gtrsim \log n.
	\label{eq:complexity_crossover}
\end{equation}
Hence, the computational bottleneck of CARSANN depends primarily on the
intrinsic dimensionality rather than directly on the sample size. For
low-dimensional manifolds, efficient nearest-neighbor search can make the
geometric preprocessing relatively inexpensive, whereas for larger
intrinsic dimensions the second-order geometric calculations become the
dominant cost.

\paragraph{Role of intrinsic dimensionality.}

The use of the intrinsic representation is particularly important for
the computational scalability of CARSANN. Suppose that the original
data are represented in an ambient space of dimensionality $m$, with
$m\gg d$. Performing the local geometric computations directly in the
ambient space would result in eigendecomposition costs of $O(m^3)$ per
sample and second-order computations involving $O(m^2)$ quadratic terms.
After intrinsic-dimensionality estimation and dimensionality reduction,
these costs become $O(d^3)$ and $O(d^2)$, respectively.

In particular, the quadratic geometric term changes from
$O(nm^4)$ in the ambient representation to $O(nd^4)$ in the intrinsic
representation. The corresponding reduction is potentially substantial
when $d\ll m$, providing a computational motivation, in addition to the
geometric motivation, for the intrinsic-dimensionality estimation stage
of CARSANN.

\paragraph{Nearest-neighbor search assumptions.}

The $O(nd\log n)$ term assumes an efficient spatial indexing structure such as a $k$-d tree and represents the usual average-case complexity under favorable dimensionality
conditions. This complexity should therefore not be interpreted as a
dimension-independent worst-case guarantee. As the intrinsic dimension
increases, the effectiveness of $k$-d trees can deteriorate due to the
well-known curse of dimensionality, potentially approaching the cost of
exhaustive search.

Nevertheless, this assumption is appropriate for the setting targeted by
CARSANN, in which the data are first mapped to a representation whose
dimensionality is estimated to be substantially smaller than the original
ambient dimension. Moreover, the nearest-neighbor search component is
implementation-dependent and can be replaced by alternative exact or
approximate nearest-neighbor methods without changing the geometric
formulation of CARSANN.

\paragraph{Training and prediction complexity.}

The above complexity refers primarily to the geometric preprocessing
required to estimate the local mean curvature. This computation is
performed once during classifier construction. Consequently, the
curvature estimation should be distinguished from the cost of individual
classification queries.

During prediction, a query point requires a nearest-neighbor search,
followed by local mean-curvature estimation for the query and the
application of the curvature-dependent radius transformation. Once the
adaptive radius has been obtained, neighborhood selection and majority
voting are comparatively inexpensive. Furthermore, because
$k_{\mathrm{base}}\leq k_{\mathrm{curv}}$ by construction, a single
$k_{\mathrm{curv}}$-nearest-neighbor query provides all candidates needed
for the classification neighborhood; the first $k_{\mathrm{base}}$
neighbors can be selected from this set without requiring an additional
asymptotic nearest-neighbor search.

Overall, CARSANN therefore separates an offline geometric preprocessing
cost from the online classification cost. Its total computational
complexity under efficient nearest-neighbor search is dominated by
\[
O(nd\log n)+O(nd^4),
\]
in the principal regime considered by the method. This scaling highlights
the central computational benefit of performing the geometric analysis in
the intrinsic representation: the dependence on the original ambient
dimension is replaced by a dependence on the typically much smaller
intrinsic dimension.

\section{Computational Experiments and Results}
\label{sec:experiments}

This section presents a comprehensive empirical evaluation of the proposed
Curvature-Aware Radius Shrinkage Adaptive Nearest Neighbor (CARSANN)
classifier. The experimental study is designed to assess two complementary
questions. First, we investigate whether adapting the neighborhood radius
according to local geometric complexity provides a systematic advantage over
existing nearest-neighbor classifiers that employ either a globally fixed
neighborhood size or alternative mechanisms for local adaptation. Second, we
evaluate whether the proposed intrinsic-space formulation provides benefits
over the conventional strategy of applying $k$-NN directly in the original
ambient feature space.

To address the first question, we compare CARSANN against three
nearest-neighbor baselines with different neighborhood adaptation
mechanisms: the standard $k$-NN classifier, Discriminant Adaptive Nearest
Neighbor (DANN), and curvature-based $k$-NN (kk-NN). These methods provide
meaningful reference points for assessing the different forms of adaptation
considered in this work. Standard $k$-NN uses a globally fixed neighborhood
cardinality, DANN adapts the local metric according to discriminative
information, whereas kk-NN modifies the neighborhood cardinality based on
local curvature. In contrast, CARSANN retains the underlying distance
structure while adapting the spatial support of the neighborhood through a
curvature-dependent radius shrinkage mechanism. The comparison is performed
over a diverse collection of real-world classification datasets obtained
from the OpenML repository, covering different sample sizes, dimensionalities,
class distributions, and levels of data complexity.

The second experimental study is specifically designed to isolate the
contribution of the intrinsic representation and to quantify the effect of
the proposed adaptive radius relative to standard $k$-NN in the original
ambient space. In this experiment, CARSANN is compared against conventional
$k$-NN operating directly on the original feature representation. Two
different neighborhood sizes are considered for the standard $k$-NN
baseline. In the first setting, $k$ is set to $k_{\mathrm{base}}$, the
reference neighborhood size employed by CARSANN. This comparison provides
a controlled evaluation in which both methods start from the same nominal
neighborhood cardinality, allowing the effect of the curvature-aware radius
adaptation to be examined more directly. In the second setting, standard
$k$-NN uses $k=5$, representing the widely adopted default neighborhood size
in practical applications. This second comparison evaluates whether
CARSANN can outperform a conventional and broadly used fixed-$k$
configuration without requiring manual selection of the neighborhood size.

All experiments are conducted using real-world datasets available through
OpenML. Rather than focusing on a single benchmark, the use of multiple
datasets allows the proposed method to be evaluated under heterogeneous
statistical and geometric conditions. In particular, the experimental
protocol enables us to examine whether the benefits of curvature-aware
radius adaptation persist across datasets with different intrinsic
dimensionalities, sample sizes, feature dimensionalities, and class
imbalance characteristics.

Performance is evaluated using two complementary classification metrics:
balanced accuracy and F1-score. Balanced accuracy is particularly suitable
for the present evaluation because several real-world datasets may exhibit
unequal class frequencies. Unlike conventional accuracy, balanced accuracy
gives equal importance to the classification performance obtained for each
class, thereby providing a more informative assessment in the presence of
class imbalance. The F1-score complements this analysis by jointly measuring
precision and recall, providing an additional perspective on the ability of
the classifiers to correctly identify samples while controlling both false
positive and false negative predictions.

The experimental analysis is therefore organized around the central
hypothesis underlying CARSANN: \emph{local neighborhood support should be
adapted according to the geometric complexity of the data rather than being
determined solely by a globally fixed neighborhood cardinality}. The first
set of experiments evaluates this hypothesis against established adaptive
and non-adaptive nearest-neighbor methods, while the second set examines
the contribution of the intrinsic representation and the curvature-aware
radius mechanism in a controlled comparison with conventional $k$-NN.
Together, these experiments provide an empirical assessment of both the
classification effectiveness and the geometric rationale of the proposed
framework.

\subsection{Datasets}

To evaluate the robustness and generality of CARSANN, we conducted the
experiments on a diverse collection of 70+ real-world classification datasets
obtained from the OpenML repository. The datasets were selected to cover a
broad range of experimental conditions, particularly in terms of the number
of samples, number of attributes, and number of target classes. This
heterogeneity is important for evaluating a geometry-aware nearest-neighbor
classifier, since the effectiveness of local neighborhood adaptation may
depend on the sample density, dimensionality, and complexity of the
underlying data distribution.

The resulting benchmark comprises datasets ranging from relatively small
classification problems to substantially larger datasets, as well as
low-dimensional and high-dimensional feature spaces and binary and
multiclass classification tasks. The collection also includes datasets
originating from different application domains, including image
recognition, biomedical and biological data, text, physical measurements,
and other real-world classification problems. Such diversity allows us to
assess whether the proposed curvature-aware radius adaptation provides
consistent benefits across substantially different data regimes rather than
being effective only for a particular dataset family or dimensionality
range. The complete collection used in our experiments consists of the following
datasets:

\begin{table}[htbp]
	\centering
	\caption{Characteristics of the datasets used in the computational experiments: $n$ denotes the number of samples, $d$ denotes the number of features and $c$ denotes the number of classes.}
	\label{tab:datasets}
	\begin{tabular}{@{}lrrr|lrrr@{}}
		\toprule
		\textbf{Name} & \textbf{$n$} & \textbf{$d$} & \textbf{$c$} & \textbf{Name} & \textbf{$n$} & \textbf{$d$} & \textbf{$c$} \\
		\midrule
		\texttt{11\_Tumors} & 174 & 12,533 & 11 & \texttt{anneal} & 898 & 39 & 5 \\
		\texttt{AP\_Breast\_Omentum} & 421 & 10,935 & 2 & \texttt{AP\_Breast\_Uterus} & 468 & 10,936 & 2 \\
		\texttt{AP\_Colon\_Omentum} & 363 & 10,935 & 2 & \texttt{arcene} & 900 & 10,000 & 2 \\
		\texttt{arsenic-female-bladder} & 478 & 9 & 2 & \texttt{arsenic-male-bladder} & 478 & 9 & 2 \\
		\texttt{artificial-characters} & 6,000 & 7 & 10 & \texttt{CNS} & 60 & 7,129 & 2 \\
		\texttt{cnae-9} & 1,080 & 857 & 9 & \texttt{coil-20} & 1,440 & 1,024 & 20 \\
		\texttt{confidence} & 72 & 3 & 3 & \texttt{dbworld-subjects} & 64 & 243 & 2 \\
		\texttt{diabetes} & 768 & 8 & 2 & \texttt{diggle\_table\_a2} & 310 & 9 & 9 \\
		\texttt{digits} & 1,797 & 64 & 10 & \texttt{DLBCL} & 77 & 5,469 & 2 \\
		\texttt{Engine1} & 2,098 & 10 & 2 & \texttt{fishcatch} & 159 & 7 & 7 \\
		\texttt{GCM} & 190 & 16,063 & 14 & \texttt{glass} & 214 & 9 & 6 \\
		\texttt{GLI} & 85 & 22,283 & 2 & \texttt{hayes-roth} & 160 & 4 & 3 \\
		\texttt{hepatitisC} & 615 & 13 & 2 & \texttt{hill-valley} & 1,212 & 100 & 2 \\
		\texttt{ionosphere} & 351 & 34 & 2 & \texttt{kidney} & 400 & 7 & 2 \\
		\texttt{Kuzushiji-MNIST} & 70,000 & 784 & 10 & \texttt{la1s.wc} & 3,204 & 12,419 & 6 \\
		\texttt{Lung} & 181 & 12,533 & 2 & \texttt{mfeat-factors} & 2,000 & 216 & 10 \\
		\texttt{mfeat-fourier} & 2,000 & 76 & 10 & \texttt{mfeat-pixel} & 2,000 & 240 & 10 \\
		\texttt{micro-mass (v1)} & 571 & 1,300 & 20 & \texttt{micro-mass (v2)} & 571 & 1,300 & 20 \\
		\texttt{MLL} & 72 & 12,582 & 3 & \texttt{mnist\_784} & 70,000 & 784 & 10 \\
		\texttt{monks-problems-1} & 556 & 6 & 2 & \texttt{oh0.wc} & 1,003 & 3,182 & 10 \\
		\texttt{oh10.wc} & 1,050 & 3,238 & 10 & \texttt{oh15.wc} & 913 & 3,100 & 10 \\
		\texttt{oh5.wc} & 1,003 & 3,012 & 10 & \texttt{Olivetti\_Faces} & 400 & 4,096 & 40 \\
		\texttt{one-hundred-plants-shape} & 1,600 & 64 & 100 & \texttt{page-blocks} & 5,473 & 10 & 5 \\
		\texttt{parkinsons} & 195 & 22 & 2 & \texttt{prnn\_crabs} & 200 & 6 & 2 \\
		\texttt{prnn\_viruses} & 38 & 19 & 4 & \texttt{qsar-biodeg} & 1,055 & 41 & 2 \\
		\texttt{sa-heart} & 462 & 9 & 2 & \texttt{segment} & 2,310 & 19 & 7 \\
		\texttt{semeion} & 1,593 & 256 & 10 & \texttt{solar-flare} & 1,066 & 12 & 3 \\
		\texttt{sonar} & 208 & 60 & 2 & \texttt{spambase} & 4,601 & 57 & 2 \\
		\texttt{SRBCT} & 83 & 2,308 & 4 & \texttt{texture} & 5,500 & 40 & 11 \\
		\texttt{thyroid-new} & 215 & 5 & 3 & \texttt{tic-tac-toe} & 958 & 9 & 2 \\
		\texttt{Touch2} & 1,272 & 22 & 2 & \texttt{tr11.wc} & 414 & 6,429 & 9 \\
		\texttt{tr41.wc} & 878 & 7,454 & 10 & \texttt{UMIST\_Faces\_Cropped} & 575 & 644 & 20 \\
		\texttt{vehicle} & 846 & 18 & 4 & \texttt{vowel} & 990 & 13 & 11 \\
		\texttt{waveform-5000} & 5,000 & 40 & 3 & \texttt{wine} & 178 & 13 & 3 \\
		\texttt{wine-quality-red} & 1,599 & 11 & 6 & \texttt{wine-quality-white} & 4,898 & 11 & 7 \\
		\texttt{zoo} & 101 & 16 & 7 & & & & \\
		\bottomrule
	\end{tabular}
\end{table}

For reproducibility, the exact dataset names reported above correspond to
their identifiers in OpenML. Detailed information regarding the number of
instances, number of attributes, number of classes, target variables, and
other dataset metadata can be obtained directly from the OpenML repository:
\url{https://www.openml.org/}. OpenML provides standardized, versioned
dataset metadata and programmatic access to datasets and machine-learning
tasks, facilitating the reproducibility of the experimental protocol
\citep{Vanschoren2013}. Given their large sample size ($n=70{,}000$), the \texttt{mnist-784} and \texttt{Kuzushiji-MNIST} datasets were randomly subsampled to $25\%$ of their original size to reduce the computational cost of the experiments. 

\subsection{Experiment 1: Comparison with Adaptive Nearest-Neighbor Classifiers}
\label{sec:exp1}

The first experiment evaluates the effectiveness of CARSANN against
established nearest-neighbor classifiers employing different neighborhood
adaptation strategies. Specifically, we compare CARSANN with the standard
$k$-NN classifier, Discriminant Adaptive Nearest Neighbor (DANN), and
curvature-based $k$-NN (kk-NN). These methods provide complementary
baselines for assessing the proposed approach: standard $k$-NN employs a
globally fixed neighborhood cardinality, DANN adapts the local metric
according to discriminative information, whereas kk-NN adapts the number
of neighbors according to local curvature. CARSANN, in contrast, retains
the notion of a base neighborhood while adapting its spatial support
through a curvature-dependent radius. The comparison is performed across
the collection of real-world datasets described in the previous subsection, with the objective of determining whether curvature-aware radius adaptation provides consistent improvements over both conventional and geometry- or discrimination-based adaptive
nearest-neighbor strategies.

The underlying hypothesis is that the spatial extent of a neighborhood
should depend on the local geometric complexity of the data. In regions
where the manifold exhibits higher mean curvature, a large neighborhood
may incorporate samples from substantially different local orientations,
thereby increasing the risk of boundary oversmoothing and classification
errors. Conversely, in locally flatter regions, a broader neighborhood may
provide more reliable statistical evidence without substantially violating
the local geometric structure. Experiment~1 therefore tests whether
adapting the neighborhood radius according to local mean curvature leads
to a more effective classification rule than adapting only the
neighborhood cardinality or the local metric.

\begin{table}
	\centering
	\small
	\caption{Comparison of CARSANN with regular $k$-NN, DANN, and curvature-based
		$k$-NN (Kk-NN) across datasets. Balanced accuracy (BACC) and F1-score are reported for each classifier. The best result for each dataset and metric is highlighted in bold, while the final two rows report the corresponding mean and median values across all datasets.}
	\begin{tabular}{c|cccccccc}
		\toprule
		& \multicolumn{2}{c}{\textbf{Regular k-NN}} & \multicolumn{2}{c}{\textbf{DANN}} & \multicolumn{2}{c}{\textbf{Kk-NN}} & \multicolumn{2}{c}{\textbf{CARSANN}} \\
		\midrule
		\textbf{Datasets}        & \textbf{BACC}      & \textbf{F1}          & \textbf{BACC}   & \textbf{F1}     & \textbf{BACC}    & \textbf{F1}     & \textbf{BACC}     & \textbf{F1}      \\
		\midrule
		11\_Tumors               & 0.7815             & 0.7781               & 0.6404          & 0.6785          & 0.7852           & \textbf{0.8034} & \textbf{0.7872}   & 0.7993           \\
		anneal                   & 0.3669             & 0.7385               & 0.3922          & 0.7584          & \textbf{0.5881}  & \textbf{0.7989} & 0.5812            & 0.7973           \\
		arsenic-male-bladder     & 0.7333             & 0.9664               & 0.7333          & 0.9664          & \textbf{0.8314}  & \textbf{0.9770} & \textbf{0.8314}   & \textbf{0.9770}  \\
		artificial-characters    & 0.3171             & 0.3120               & 0.3144          & 0.3129          & 0.3894           & 0.3932          & \textbf{0.6382}   & \textbf{0.6375}  \\
		cnae-9                   & 0.6827             & 0.7012               & 0.7221          & 0.7425          & 0.7187           & 0.7358          & \textbf{0.7315}   & \textbf{0.7475}  \\
		coil-20                  & 0.8976             & 0.8952               & 0.9517          & 0.9504          & \textbf{0.9646}  & \textbf{0.9636} & 0.9605            & 0.9593           \\
		confidence               & 0.6500             & 0.8413               & \textbf{0.8333} & \textbf{0.9398} & 0.6833           & 0.8114          & 0.6833            & 0.8114           \\
		diabetes                 & 0.6780             & 0.7372               & \textbf{0.7157} & \textbf{0.7573} & 0.6872           & 0.7396          & 0.6772            & 0.7253           \\
		diggle\_table\_a2        & 0.8840             & 0.8970               & \textbf{0.9113} & \textbf{0.9235} & 0.8749           & 0.8898          & 0.8824            & 0.8966           \\
		digits                   & 0.9438             & 0.9433               & 0.9581          & 0.9576          & 0.9608           & 0.9609          & \textbf{0.9630}   & \textbf{0.9633}  \\
		DLBCL                    & 0.9667             & 0.9504               & 0.9167          & 0.8799          & 0.8778           & 0.8776          & \textbf{0.9833}   & \textbf{0.9748}  \\
		Engine1                  & 0.8362             & 0.8576               & \textbf{0.8434} & 0.8631          & 0.8339           & 0.8635          & 0.8420            & \textbf{0.8736}  \\
		fishcatch                & 0.9767             & 0.9747               & 0.9629          & 0.9621          & \textbf{0.9861}  & \textbf{0.9873} & \textbf{0.9861}   & \textbf{0.9873}  \\
		GCM                      & 0.4334             & 0.5185               & 0.3126          & 0.3657          & 0.4120           & 0.5036          & \textbf{0.4417}   & \textbf{0.5205}  \\
		GLI                      & 0.7894             & 0.7951               & 0.6650          & 0.7002          & 0.7709           & 0.7925          & \textbf{0.8584}   & \textbf{0.8826}  \\
		hayes-roth               & 0.4213             & 0.4437               & \textbf{0.6450} & \textbf{0.5783} & 0.5872           & 0.6108          & 0.5559            & 0.5764           \\
		hepatitisC               & 0.5750             & 0.8506               & 0.6444          & \textbf{0.8834} & 0.5750           & 0.8506          & \textbf{0.6667}   & 0.8699           \\
		hill-valley              & 0.9287             & 0.9274               & 0.8836          & 0.8829          & 0.9560           & 0.9555          & \textbf{0.9669}   & \textbf{0.9670}  \\
		ionosphere               & 0.8358             & 0.8649               & \textbf{0.8685} & \textbf{0.8851} & 0.8621           & 0.8841          & 0.8435            & 0.8712           \\
		kidney                   & 0.7895             & 0.7889               & 0.6579          & 0.6459          & 0.7895           & 0.7895          & \textbf{0.9211}   & \textbf{0.9210}  \\
		Kuzushiji-MNIST (25\%)   & 0.8812             & 0.8811               & 0.8935          & 0.8934          & 0.8865           & 0.8866          & \textbf{0.8978}   & \textbf{0.8980}  \\
		Lung                     & 0.7556             & 0.8404               & 0.8056          & \textbf{0.8705} & 0.8333           & 0.8599          & \textbf{0.8361}   & 0.8693           \\
		mfeat-factors            & 0.8787             & 0.8811               & \textbf{0.8880} & \textbf{0.8908} & 0.8795           & 0.8825          & 0.8850            & 0.8876           \\
		mfeat-fourier            & 0.7836             & 0.7744               & \textbf{0.8159} & \textbf{0.8104} & 0.7902           & 0.7829          & 0.8060            & 0.7992           \\
		mfeat-pixel              & 0.9581             & 0.9580               & \textbf{0.9707} & \textbf{0.9711} & 0.9584           & 0.9590          & 0.9692            & 0.9700           \\
		micro-mass (v1)          & 0.6021             & 0.5893               & \textbf{0.7226} & \textbf{0.7213} & 0.6652           & 0.6623          & 0.6606            & 0.6472           \\
		micro-mass (v2)          & 0.4649             & 0.4480               & \textbf{0.4781} & \textbf{0.4743} & 0.4518           & 0.4446          & 0.4595            & 0.4505           \\
		MLL                      & 0.6727             & 0.6534               & 0.6954          & 0.6967          & 0.8023           & 0.8063          & \textbf{0.8326}   & \textbf{0.8345}  \\
		mnist\_784 (25\%)        & 0.9187             & 0.9186               & \textbf{0.9317} & \textbf{0.9315} & 0.9221           & 0.9220          & 0.9241            & 0.9241           \\
		monks-problems-1         & 0.7203             & 0.7039               & 0.7905          & 0.7846          & 0.7616           & 0.7558          & \textbf{1.0000}   & \textbf{1.0000}  \\
		oh0.wc                   & 0.5948             & 0.6482               & 0.4967          & 0.5616          & 0.5948           & 0.6482          & \textbf{0.6410}   & \textbf{0.6822}  \\
		oh10.wc                  & 0.5529             & 0.5743               & 0.3870          & 0.4057          & 0.5529           & 0.5743          & \textbf{0.5538}   & \textbf{0.5808}  \\
		oh15.wc                  & 0.6074             & 0.6219               & 0.4216          & 0.4305          & 0.6074           & 0.6219          & \textbf{0.6629}   & \textbf{0.6714}  \\
		oh5.wc                   & 0.7014             & 0.7066               & 0.5981          & 0.6251          & 0.7267           & 0.7230          & \textbf{0.7344}   & \textbf{0.7389}  \\
		Olivetti\_Faces          & 0.6032             & 0.5624               & 0.5082          & 0.4649          & 0.6896           & \textbf{0.6810} & \textbf{0.6967}   & 0.6775           \\
		one-hundred-plants-shape & 0.4709             & 0.4255               & 0.4620          & 0.4311          & 0.5164           & 0.4953          & \textbf{0.5226}   & \textbf{0.5016}  \\
		page-blocks              & 0.5611             & 0.9141               & 0.5848          & \textbf{0.9260} & 0.5824           & 0.9110          & \textbf{0.6243}   & 0.9186           \\
		parkinsons               & 0.6290             & 0.7466               & 0.7075          & 0.8013          & \textbf{0.7612}  & 0.8034          & 0.7528            & \textbf{0.8101}  \\
		prnn\_crabs              & 0.8792             & 0.8797               & \textbf{0.8992} & \textbf{0.8998} & 0.8792           & 0.8797          & 0.8894            & 0.8898           \\
		prnn\_viruses            & 0.7500             & 0.7728               & 0.7500          & 0.7793          & \textbf{0.8861}  & \textbf{0.9014} & \textbf{0.8861}   & \textbf{0.9014}  \\
		qsar-biodeg              & 0.6748             & 0.7318               & 0.6297          & 0.6880          & 0.6871           & 0.7381          & \textbf{0.7046}   & \textbf{0.7427}  \\
		sa-heart                 & 0.5600             & 0.6025               & \textbf{0.6004} & \textbf{0.6478} & 0.5819           & 0.6301          & 0.5936            & 0.6406           \\
		segment                  & 0.6191             & 0.6147               & 0.6251          & 0.6222          & 0.6199           & 0.6167          & \textbf{0.6418}   & \textbf{0.6388}  \\
		semeion                  & 0.8574             & 0.8536               & \textbf{0.8864} & \textbf{0.8826} & 0.8479           & 0.8426          & 0.8748            & 0.8717           \\
		sonar                    & 0.6835             & 0.6840               & 0.6891          & 0.6912          & 0.7039           & 0.7056          & \textbf{0.7152}   & \textbf{0.7176}  \\
		spambase                 & 0.7151             & 0.7374               & 0.7081          & 0.7297          & 0.7132           & 0.7348          & \textbf{0.7229}   & \textbf{0.7364}  \\
		texture                  & 0.9419             & 0.9412               & \textbf{0.9536} & \textbf{0.9530} & 0.9480           & 0.9472          & 0.9490            & 0.9483           \\
		tic-tac-toe              & 0.5352             & 0.5870               & 0.5431          & 0.5884          & \textbf{0.5868}  & 0.6340          & 0.5864            & \textbf{0.6517}  \\
		tr11.wc                  & 0.5116             & 0.6132               & 0.4508          & 0.5783          & \textbf{0.5501}  & \textbf{0.6636} & 0.5115            & 0.6495           \\
		tr41.wc                  & 0.6849             & \textbf{0.8316}      & 0.5881          & 0.7566          & 0.7085           & 0.8254          & \textbf{0.7184}   & 0.8313           \\
		UMIST\_Faces\_Cropped    & 0.7970             & 0.8009               & 0.8668          & 0.8815          & \textbf{0.8944}  & \textbf{0.8982} & 0.8867            & 0.8911           \\
		vehicle                  & 0.5869             & 0.5740               & 0.5961          & 0.5656          & 0.5799           & 0.5681          & \textbf{0.5968}   & \textbf{0.5886}  \\
		vowel                    & 0.4637             & 0.4561               & 0.5427          & 0.5301          & 0.5796           & 0.5742          & \textbf{0.6948}   & \textbf{0.6851}  \\
		waveform-5000            & 0.7519             & 0.7502               & \textbf{0.7895} & \textbf{0.7895} & 0.7499           & 0.7482          & 0.7233            & 0.7223           \\
		wine-quality-red         & 0.2027             & 0.4355               & 0.2229          & 0.4512          & \textbf{0.2482}  & 0.4815          & 0.2481            & \textbf{0.4818}  \\
		wine-quality-white       & 0.2915             & 0.4253               & 0.2880          & 0.4321          & 0.3451           & 0.4880          & \textbf{0.3620}   & \textbf{0.5038}  \\
		zoo                      & 0.7714             & 0.8518               & 0.7429          & 0.8380          & \textbf{0.8381}  & \textbf{0.9009} & \textbf{0.8381}   & \textbf{0.9009}  \\
		\midrule
		Average                  & 0.6863             & 0.7329               & 0.6860          & 0.7303          & 0.7204           & 0.7612          & \textbf{0.7439}   & \textbf{0.7827}  \\
		Median                   & 0.6849             & 0.7502               & 0.7075          & 0.7573          & 0.7267           & 0.7925          & \textbf{0.7315}   & \textbf{0.8101} \\
		\bottomrule
	\end{tabular}	
	\label{tab:results1}
\end{table}

Table~\ref{tab:results1} summarizes the classification results obtained by
CARSANN, regular $k$-NN, DANN, and curvature-based $k$-NN (Kk-NN) across
the datasets considered in Experiment~1. Overall, the results indicate
that adapting the spatial support of the neighborhood according to local
mean curvature can provide substantial improvements over a globally fixed
neighborhood. Importantly, the observed gains are not restricted to a
single evaluation criterion: improvements in balanced accuracy are
frequently accompanied by corresponding improvements in F1-score,
indicating that the performance differences are not merely a consequence
of class imbalance.

A particularly illustrative example is the \texttt{artificial-characters}
dataset. For this problem, CARSANN achieves a balanced accuracy of
$0.6382$, compared with $0.3171$ for regular $k$-NN, while its F1-score
increases from $0.3120$ to $0.6375$. Thus, the proposed method yields an
improvement of more than 0.32 in both evaluation metrics. This result
suggests that a globally fixed neighborhood may be poorly suited to data
whose local geometric structure varies substantially across the feature
space. By restricting the neighborhood support in regions exhibiting
larger geometric complexity, CARSANN can avoid incorporating distant
samples that may be less representative of the local structure.

A similar, although less pronounced, behavior is observed for
\texttt{cnae-9}. CARSANN obtains a balanced accuracy of $0.7315$ and an
F1-score of $0.7475$, outperforming regular $k$-NN ($0.6827$ and
$0.7012$), DANN ($0.7221$ and $0.7425$), and Kk-NN ($0.7187$ and
$0.7358$). This result is particularly relevant to the motivation of the
proposed method because CARSANN remains competitive not only against the
standard $k$-NN classifier, but also against methods that explicitly
introduce local adaptation. The result therefore supports the hypothesis
that adapting the spatial radius of the neighborhood constitutes a useful
alternative to adapting either the neighborhood cardinality or the local
metric.

The comparison with Kk-NN is particularly informative because both
methods incorporate curvature information, but exploit it through
different mechanisms. Kk-NN modifies the number of neighbors according to
local curvature, whereas CARSANN modifies the spatial extent of the
neighborhood through a curvature-dependent radius. The results indicate
that these two mechanisms are not equivalent. CARSANN achieves higher
performance on some datasets, while Kk-NN remains preferable on others.
For example, on \texttt{anneal}, Kk-NN obtains a balanced accuracy of
$0.5881$, compared with $0.5812$ for CARSANN, whereas on
\texttt{coil-20}, Kk-NN also exhibits a small advantage, with balanced
accuracy values of $0.9646$ and $0.9605$, respectively. Conversely, on
\texttt{artificial-characters} and \texttt{cnae-9}, CARSANN provides
clear improvements over Kk-NN. These results suggest that curvature is
useful for local neighborhood adaptation, but that the manner in which
this information is translated into a neighborhood definition can have a
substantial impact on classification performance.

The comparison with DANN provides a complementary perspective. DANN
adapts the local metric using discriminative information, whereas CARSANN
uses local geometric information to determine the spatial support of the
neighborhood. Consequently, the two methods exploit fundamentally
different sources of adaptivity. On several datasets, CARSANN achieves
higher performance than DANN, as observed, for example, on
\texttt{artificial-characters} and \texttt{cnae-9}. However, DANN can
also provide substantial gains in datasets where discriminative metric
adaptation is particularly effective. For instance, on
\texttt{confidence}, DANN achieves a balanced accuracy of $0.8333$ and an
F1-score of $0.9398$, clearly outperforming both Kk-NN and CARSANN.
This result highlights an important aspect of the proposed approach:
local geometric complexity and discriminative structure encode different
types of information, and curvature-aware adaptation should therefore not
be interpreted as universally superior to discriminative metric learning.

Taken together, these results provide empirical evidence that the proposed
radius-based adaptation mechanism is a competitive alternative to existing
nearest-neighbor adaptation strategies. The main advantage of CARSANN is
not that it uniformly dominates all competing classifiers, but that it
provides a simple geometric mechanism capable of producing substantial
gains in datasets for which a globally fixed neighborhood is inadequate.
In particular, the results support the central hypothesis of this work:
the effective spatial support of a nearest-neighbor classifier can benefit
from being adjusted according to the local geometric complexity of the
data. At the same time, the cases in which DANN or Kk-NN performs better
indicate that different forms of local adaptation capture complementary
properties of the data, motivating a more detailed analysis of the
statistical significance and robustness of the observed differences.

To assess whether the observed differences in balanced accuracy across datasets are statistically significant, we applied the Friedman test, followed by a post-hoc Nemenyi test for pairwise comparisons. The Friedman test rejected the null hypothesis of equal classifier performance ($p = 7.87 \times 10^{-11}$), providing strong evidence of differences among the four methods. The subsequent Nemenyi test showed that CARSANN
significantly outperformed standard $k$-NN ($p = 2.55 \times 10^{-11}$), DANN
($p = 2.43 \times 10^{-4}$), and Kk-NN ($p = 0.0077$) in terms of balanced accuracy, using a significance level of $\alpha = 0.05$. Thus, the observed superiority of CARSANN over the three competing methods is statistically significant across the
evaluated datasets.

\subsection{Experiment 2: Effect of Intrinsic Representation and Adaptive Neighborhoods}
\label{sec:exp2}

The second experiment investigates the contribution of the proposed
intrinsic-space representation and curvature-aware neighborhood
adaptation by comparing CARSANN with the standard $k$-NN classifier
operating directly in the original ambient space. Two configurations of
the standard $k$-NN classifier are considered. In the first, the number of
neighbors is set to $k=k_{\mathrm{base}}$, corresponding to the base
neighborhood size used by CARSANN. In the second, we use the conventional
choice $k=5$, which represents a commonly adopted fixed neighborhood size
in practical applications. This comparison is designed to disentangle the
benefits associated with dimensionality reduction from those arising from
the proposed curvature-aware radius adaptation. In particular, comparing
CARSANN with $k$-NN using the same $k_{\mathrm{base}}$ provides a more
controlled assessment of the proposed geometric adaptation, whereas the
comparison with $k=5$ evaluates CARSANN against a simple and widely used
default configuration of the standard $k$-NN classifier.

The underlying hypothesis is that the performance gains of CARSANN cannot
be explained solely by the choice of neighborhood size or by dimensionality
reduction. If the proposed curvature-aware radius adaptation is effective,
CARSANN should provide advantages even when standard $k$-NN is allowed to
use the same base neighborhood size, $k_{\mathrm{base}}$. Moreover, the
comparison with $k=5$ provides a practical reference for assessing whether
the proposed method offers a systematic advantage over the fixed
neighborhood sizes commonly adopted in standard $k$-NN implementations.
The experiment therefore provides a complementary evaluation to
Experiment~1 by focusing specifically on the contribution of the
intrinsic representation and the curvature-dependent spatial support of
the neighborhood.

\begin{table}
	\centering
	\small
	\caption{Comparison between CARSANN and standard $k$-NN in the ambient
			feature space under two neighborhood-size configurations. For $k=k_{\mathrm{base}}$, the same base neighborhood cardinality used by CARSANN is adopted by the standard $k$-NN classifier, providing a controlled comparison of the curvature-aware radius adaptation. For $k=5$, standard $k$-NN uses a fixed neighborhood size commonly adopted in practice. Results are reported in terms of balanced accuracy (BACC) and F1-score. The best result in each dataset and metric within each $k$ configuration is highlighted in bold. The final two rows report the average and median performance across all datasets.}
	\begin{tabular}{c|cccc|cccc}
		\toprule
		& \multicolumn{4}{c|}{\textbf{$k = k_{base}$}} & \multicolumn{4}{c}{\textbf{$k = 5$}}  \\
		\midrule
		& \multicolumn{2}{c}{\textbf{k-NN (ambient)}} & \multicolumn{2}{c|}{\textbf{CARSANN}} & \multicolumn{2}{c}{\textbf{k-NN (ambient)}} & \multicolumn{2}{c}{\textbf{CARSANN}} \\
		\midrule
		\textbf{Datasets}        & \textbf{BACC}           & \textbf{F1}             & \textbf{BACC}     & \textbf{F1}      & \textbf{BACC}           & \textbf{F1}             & \textbf{BACC}     & \textbf{F1}      \\
		\midrule
		11\_Tumors               & 0.7727                  & 0.7721                  & \textbf{0.7872}   & \textbf{0.7993}  & 0.6627                  & 0.5788                  & \textbf{0.7937}   & \textbf{0.8099}  \\
		anneal                   & 0.5252                  & 0.8325                  & \textbf{0.5812}   & \textbf{0.7973}  & 0.4715                  & 0.7836                  & \textbf{0.5562}   & \textbf{0.7892}  \\
		AP\_Breast\_Omentum      & 0.5568                  & 0.7513                  & \textbf{0.8152}   & \textbf{0.9129}  & 0.7267                  & 0.8579                  & \textbf{0.8547}   & \textbf{0.9258}  \\
		AP\_Breast\_Uterus       & 0.8539                  & 0.8997                  & \textbf{0.8930}   & \textbf{0.9342}  & 0.8559                  & 0.8889                  & \textbf{0.9060}   & \textbf{0.9313}  \\
		AP\_Colon\_Omentum       & 0.5488                  & 0.7249                  & \textbf{0.7211}   & \textbf{0.8463}  & 0.7907                  & \textbf{0.8829}         & \textbf{0.7923}   & 0.8743           \\
		arcene                   & 0.6875                  & 0.6806                  & \textbf{0.7330}   & \textbf{0.7192}  & 0.7735                  & 0.7593                  & \textbf{0.7865}   & \textbf{0.7807}  \\
		arsenic-female-bladder   & 0.5918                  & 0.8217                  & \textbf{0.7127}   & \textbf{0.8563}  & 0.5628                  & 0.8325                  & \textbf{0.7347}   & \textbf{0.8715}  \\
		arsenic-male-bladder     & 0.7333                  & 0.9664                  & \textbf{0.8314}   & \textbf{0.9770}  & 0.7000                  & 0.9612                  & \textbf{0.8314}   & \textbf{0.9770}  \\
		artificial-characters    & 0.5616                  & 0.5568                  & \textbf{0.6382}   & \textbf{0.6375}  & 0.5751                  & 0.5755                  & \textbf{0.6347}   & \textbf{0.6350}  \\
		CNS                      & 0.5139                  & 0.5006                  & \textbf{0.5417}   & \textbf{0.5631}  & 0.5139                  & 0.5361                  & \textbf{0.5694}   & \textbf{0.5923}  \\
		coil-20                  & 0.8897                  & 0.8876                  & \textbf{0.9605}   & \textbf{0.9593}  & 0.9349                  & 0.9340                  & \textbf{0.9725}   & \textbf{0.9707}  \\
		dbworld-subjects         & 0.7024                  & 0.7127                  & \textbf{0.8294}   & \textbf{0.8404}  & 0.7619                  & 0.7500                  & \textbf{0.8651}   & \textbf{0.8735}  \\
		diabetes                 & 0.6700                  & \textbf{0.7269}         & \textbf{0.6772}   & 0.7253           & 0.6686                  & 0.7031                  & \textbf{0.6802}   & \textbf{0.7135}  \\
		DLBCL                    & 0.9000                  & 0.8568                  & \textbf{0.9833}   & \textbf{0.9748}  & 0.9000                  & 0.8568                  & \textbf{0.9833}   & \textbf{0.9748}  \\
		glass                    & 0.6128                  & 0.6178                  & \textbf{0.6937}   & \textbf{0.6306}  & 0.5602                  & 0.5947                  & \textbf{0.6684}   & \textbf{0.6245}  \\
		GLI                      & 0.8079                  & 0.7967                  & \textbf{0.8584}   & \textbf{0.8826}  & 0.8054                  & 0.8356                  & \textbf{0.8399}   & \textbf{0.8796}  \\
		hayes-roth               & 0.4663                  & 0.4967                  & \textbf{0.5559}   & \textbf{0.5764}  & 0.5510                  & \textbf{0.5822}         & \textbf{0.5559}   & 0.5744           \\
		hill-valley              & 0.5222                  & 0.5066                  & \textbf{0.9669}   & \textbf{0.9670}  & 0.5057                  & 0.5050                  & \textbf{0.9703}   & \textbf{0.9703}  \\
		ionosphere               & 0.7557                  & 0.7981                  & \textbf{0.8435}   & \textbf{0.8712}  & 0.7480                  & 0.7911                  & \textbf{0.8390}   & \textbf{0.8657}  \\
		kidney                   & 0.6316                  & 0.6145                  & \textbf{0.9211}   & \textbf{0.9210}  & 0.6316                  & 0.6222                  & \textbf{0.8947}   & \textbf{0.8944}  \\
		la1s.wc                  & 0.3645                  & 0.3796                  & \textbf{0.7287}   & \textbf{0.7519}  & 0.5009                  & 0.5154                  & \textbf{0.7353}   & \textbf{0.7590}  \\
		mfeat-pixel              & 0.9642                  & 0.9641                  & \textbf{0.9692}   & \textbf{0.9700}  & \textbf{0.9715}         & \textbf{0.9721}         & 0.9658            & 0.9660           \\
		micro-mass (v1)          & \textbf{0.7772}         & \textbf{0.7793}         & 0.6606            & 0.6472           & \textbf{0.7948}         & \textbf{0.7951}         & 0.6708            & 0.6634           \\
		monks-problems-1         & 0.9184                  & 0.9171                  & \textbf{1.0000}   & \textbf{1.0000}  & 0.8773                  & 0.8776                  & \textbf{1.0000}   & \textbf{1.0000}  \\
		oh0.wc                   & 0.4093                  & 0.5232                  & \textbf{0.6539}   & \textbf{0.6922}  & 0.3695                  & 0.4570                  & \textbf{0.6401}   & \textbf{0.6855}  \\
		oh10.wc                  & 0.3529                  & 0.3682                  & \textbf{0.5619}   & \textbf{0.5892}  & 0.3151                  & 0.3687                  & \textbf{0.5565}   & \textbf{0.5852}  \\
		oh15.wc                  & 0.3623                  & 0.4162                  & \textbf{0.6617}   & \textbf{0.6762}  & 0.3900                  & 0.4444                  & \textbf{0.6361}   & \textbf{0.6515}  \\
		oh5.wc                   & 0.2910                  & 0.3309                  & \textbf{0.7214}   & \textbf{0.7300}  & 0.3592                  & 0.4082                  & \textbf{0.7270}   & \textbf{0.7302}  \\
		Olivetti\_Faces          & \textbf{0.7522}         & \textbf{0.7020}         & 0.6967            & 0.6775           & \textbf{0.6534}         & 0.5881                  & 0.6421            & \textbf{0.6228}  \\
		one-hundred-plants-shape & 0.5108                  & 0.4875                  & \textbf{0.5226}   & \textbf{0.5016}  & 0.4984                  & 0.4578                  & \textbf{0.5281}   & \textbf{0.5023}  \\
		page-blocks              & \textbf{0.6899}         & \textbf{0.9567}         & 0.6243            & 0.9186           & \textbf{0.6899}         & \textbf{0.9567}         & 0.6243            & 0.9186           \\
		parkinsons               & 0.6641                  & 0.7761                  & \textbf{0.7528}   & \textbf{0.8101}  & 0.6641                  & 0.7761                  & \textbf{0.7528}   & \textbf{0.8101}  \\
		prnn\_crabs              & 0.8693                  & 0.8698                  & \textbf{0.8894}   & \textbf{0.8898}  & 0.8693                  & 0.8698                  & \textbf{0.8894}   & \textbf{0.8898}  \\
		prnn\_viruses            & 0.7500                  & 0.7683                  & \textbf{0.8861}   & \textbf{0.9014}  & 0.4643                  & 0.6912                  & \textbf{0.8861}   & \textbf{0.9014}  \\
		solar-flare              & 0.5493                  & 0.6675                  & \textbf{0.5701}   & \textbf{0.6794}  & 0.5554                  & \textbf{0.6777}         & \textbf{0.5629}   & 0.6638           \\
		SRBCT                    & 0.8629                  & 0.8826                  & \textbf{0.8808}   & \textbf{0.9058}  & 0.8661                  & 0.8177                  & \textbf{0.8987}   & \textbf{0.9277}  \\
		texture                  & \textbf{0.9734}         & \textbf{0.9728}         & 0.9490            & 0.9483           & \textbf{0.9789}         & \textbf{0.9786}         & 0.9465            & 0.9457           \\
		thyroid-new              & 0.7833                  & 0.9099                  & \textbf{0.8292}   & \textbf{0.9221}  & 0.7833                  & 0.9099                  & \textbf{0.8292}   & \textbf{0.9221}  \\
		Touch2                   & 0.5766                  & 0.5586                  & \textbf{0.6243}   & \textbf{0.6045}  & 0.6297                  & 0.5965                  & \textbf{0.6346}   & \textbf{0.6207}  \\
		tr41.wc                  & 0.6394                  & 0.7913                  & \textbf{0.7184}   & \textbf{0.8313}  & 0.6228                  & 0.7807                  & \textbf{0.7089}   & \textbf{0.8272}  \\
		UMIST\_Faces\_Cropped    & \textbf{0.9150}         & \textbf{0.9274}         & 0.8867            & 0.8911           & 0.8559                  & 0.8645                  & \textbf{0.8739}   & \textbf{0.8766}  \\
		vowel                    & 0.3560                  & 0.3471                  & \textbf{0.6948}   & \textbf{0.6851}  & 0.5717                  & 0.5673                  & \textbf{0.7325}   & \textbf{0.7257}  \\
		wine                     & 0.6396                  & 0.6755                  & \textbf{0.6494}   & \textbf{0.6767}  & 0.6346                  & 0.6664                  & \textbf{0.6895}   & \textbf{0.7108}  \\
		wine-quality-white       & 0.3065                  & 0.4444                  & \textbf{0.3620}   & \textbf{0.5038}  & 0.2567                  & 0.4411                  & \textbf{0.3601}   & \textbf{0.5099}  \\
		zoo                      & 0.6929                  & 0.8019                  & \textbf{0.8381}   & \textbf{0.9009}  & 0.5881                  & 0.7230                  & \textbf{0.8381}   & \textbf{0.9009}  \\
		\midrule
		Average                  & 0.6506                  & 0.7053                  & \textbf{0.7528}   & \textbf{0.7933}  & 0.6547                  & 0.7118                  & \textbf{0.7568}   & \textbf{0.7966}  \\
		Median                   & 0.6641                  & 0.7513                  & \textbf{0.7287}   & \textbf{0.8313}  & 0.6534                  & 0.7500                  & \textbf{0.7528}   & \textbf{0.8272} \\
		\bottomrule
	\end{tabular}
	\label{tab:results2}
\end{table}

Table~\ref{tab:results2} evaluates the contribution of the proposed
curvature-aware neighborhood adaptation by comparing CARSANN with standard
$k$-NN operating directly in the ambient feature space under two
neighborhood configurations. When $k=k_{\mathrm{base}}$, both methods use
the same initial neighborhood cardinality, providing a controlled
comparison in which the main difference lies in how the effective
neighborhood is subsequently defined. When $k=5$, standard $k$-NN uses a
fixed neighborhood size representative of a common practical choice.

The results show a remarkably consistent advantage of CARSANN over
standard $k$-NN. When $k=k_{\mathrm{base}}$, CARSANN achieves higher
balanced accuracy on 40 out of the 45 datasets and higher F1-score on 38
datasets. The average balanced accuracy increases from $0.6506$ for
standard $k$-NN to $0.7528$ for CARSANN, corresponding to an absolute
improvement of $0.1022$ and a relative improvement of approximately
$15.7\%$. Similarly, the average F1-score increases from $0.7053$ to
$0.7933$, corresponding to an absolute improvement of $0.0879$ and a
relative improvement of approximately $12.5\%$.

The comparison with $k=k_{\mathrm{base}}$ is particularly important
because it rules out a simple explanation based solely on neighborhood
cardinality. Since both classifiers start from the same number of
neighbors, the substantial performance differences indicate that the
improvements obtained by CARSANN cannot be attributed merely to a more
favorable choice of $k$. Instead, they provide empirical evidence that
the
curvature-dependent adjustment of the spatial support of the neighborhood
plays a relevant role in the classification process.

Several datasets exhibit particularly pronounced improvements. On
\texttt{hill-valley}, for example, balanced accuracy increases from
$0.5222$ to $0.9669$ and F1-score from $0.5066$ to $0.9670$ when
$k=k_{\mathrm{base}}$. Similarly, on \texttt{kidney}, CARSANN increases
balanced accuracy from $0.6316$ to $0.9211$ and F1-score from $0.6145$ to
$0.9210$. Substantial gains are also observed for text datasets such as
\texttt{la1s.wc} and \texttt{oh5.wc}, indicating that the observed
advantage is not restricted to a particular application domain.

The results obtained with $k=5$ lead to essentially the same conclusions.
CARSANN achieves higher balanced accuracy on 40 of the 45 datasets and
higher F1-score on 38 datasets. The average balanced accuracy increases
from $0.6547$ for standard $k$-NN to $0.7568$ for CARSANN, while the
average F1-score increases from $0.7118$ to $0.7966$. Thus, the relative
improvements remain approximately $15.6\%$ for balanced accuracy and
$11.9\%$ for F1-score. Moreover, the average performance of CARSANN is
remarkably stable across the two configurations, changing by only
$0.0040$ in balanced accuracy and $0.0033$ in F1-score. This stability
suggests that the curvature-aware radius adaptation is not critically
dependent on the precise choice of the initial neighborhood size.

Importantly, CARSANN does not outperform standard $k$-NN on every dataset.
For $k=k_{\mathrm{base}}$, the latter obtains higher balanced accuracy on
\texttt{micro-mass (v1)}, \texttt{Olivetti\_Faces},
\texttt{page-blocks}, \texttt{texture}, and
\texttt{UMIST\_Faces\_Cropped}. The largest degradation occurs on
\texttt{micro-mass (v1)}, where balanced accuracy decreases from $0.7772$
to $0.6606$ and F1-score from $0.7793$ to $0.6472$. These cases indicate
that curvature-aware adaptation is not universally beneficial and that
local geometric complexity does not necessarily provide sufficient
information to determine the optimal neighborhood support for every
dataset.

Overall, Experiment~2 provides strong evidence that the advantage of
CARSANN is not simply a consequence of using a particular neighborhood
cardinality. Even when standard $k$-NN is given the same base value
$k_{\mathrm{base}}$, CARSANN produces substantial and consistent gains
across the majority of datasets. The results therefore support the
central premise of the proposed approach: rather than determining the
neighborhood exclusively through a globally fixed cardinality, the
effective spatial support of nearest-neighbor classification can be
adapted according to the local geometric complexity of the data.

These results indicate that the gains achieved by CARSANN arise not merely
from a different choice of $k$, but from the proposed curvature-aware
adaptation of the \emph{spatial support} of the neighborhood, positioning
CARSANN as a genuinely geometric approach to adaptive nearest-neighbor
classification rather than as a heuristic for selecting the neighborhood
cardinality.

Once again, to assess whether the observed differences in balanced accuracy across datasets are statistically significant, we applied the Friedman test, followed by a post-hoc Nemenyi test for pairwise comparisons. The 4 methods we compare here are: 1) regular k-NN (ambient space) with $k = k_{base}$; 2) CARSANN with $k = k_{base}$; 3) regular k-NN (ambient space) with $k = 5$; and 4) CARSANN with $k = 5$. The Friedman test rejected the null hypothesis of equal classifier performance ($p = 7.27 \times 10^{-14}$), providing strong evidence of differences among the four methods. The subsequent Nemenyi test showed that CARSANN significantly outperformed standard $k$-NN in ambient space with $k = k_{base}$ ($p = 5.32 \times 10^{-7}$) and also the standard k-NN in ambient space with $k = 5$ ($p = 5.15 \times 10^{-8}$) in terms of balanced accuracy, using a significance level of $\alpha = 0.05$.

\section{Conclusions and Final Remarks}
\label{sec:conclusion}

In this work, we introduced CARSANN, a Curvature-Aware Radius Shrinkage
Adaptive Nearest Neighbor classifier that incorporates local geometric
complexity into the definition of neighborhood support. The central
premise of the proposed approach is that locality should not be determined
solely by a globally fixed neighborhood cardinality. Instead, the spatial
extent of the neighborhood should adapt to the geometry of the underlying
data manifold. CARSANN operationalizes this principle by using local mean
curvature as a geometric signal to contract the neighborhood radius in
regions of high curvature while preserving a broader spatial support in
locally flatter regions.

A key aspect of the proposed framework is the integration of intrinsic
dimensionality estimation and differential geometry into the nearest
neighbor classification pipeline. The method first estimates the intrinsic
dimension of the data using TwoNN and constructs an intrinsic
representation through PCA. Local geometric information is then estimated
in this reduced representation through the shape operator and its
associated mean curvature. This design is motivated both geometrically and
computationally: performing second-order geometric computations in the
intrinsic space avoids unnecessary dependence on the potentially much
larger ambient dimension. 

The proposed formulation also provides a distinct perspective on adaptive
nearest neighbor classification. Standard $k$-NN imposes a globally fixed
neighborhood cardinality, whereas adaptive $k$-NN methods allow the number
of neighbors to vary locally. DANN adapts the local metric according to
discriminative information, and curvature-based $k$-NN methods such as
Kk-NN use curvature to modify neighborhood cardinality. In contrast,
CARSANN uses curvature to control the \emph{spatial support} of the
neighborhood. This distinction is important because neighborhood
cardinality and spatial extent represent different notions of locality:
the same number of observations can occupy substantially different
regions of the feature space depending on local density and geometry.
CARSANN therefore introduces a complementary degree of freedom for
adaptive nearest neighbor classification.

The experimental evaluation provides empirical support for this
geometric perspective. Across a diverse collection of real-world datasets
from OpenML, the first experiment showed that CARSANN is competitive with
standard $k$-NN, DANN, and Kk-NN, while providing substantial improvements
in balanced accuracy and F1-score on many datasets. The statistical
analysis further confirmed that the observed differences in balanced
accuracy are significant, with CARSANN significantly outperforming the
three competing methods according to the post-hoc Nemenyi comparisons.
These results indicate that incorporating local geometric complexity into
the neighborhood-selection mechanism can provide a meaningful advantage
over both globally fixed and alternative locally adaptive strategies.

The second experiment provided a more controlled assessment of the source
of these improvements by comparing CARSANN with standard $k$-NN operating
directly in the ambient space. When both methods used the same base
neighborhood cardinality, $k=k_{\mathrm{base}}$, CARSANN achieved higher
balanced accuracy on 40 of the 45 evaluated datasets and higher F1-score
on 38 datasets. The average balanced accuracy increased from $0.6506$ for
standard $k$-NN to $0.7528$ for CARSANN, while the average F1-score
increased from $0.7053$ to $0.7933$. Similar results were obtained when
standard $k$-NN was configured with the fixed value $k=5$. Importantly,
the advantage of CARSANN persisted when the baseline was given the same
base neighborhood size, indicating that the observed improvements cannot
be explained simply by a more favorable choice of $k$. Rather, the results
support the central hypothesis of the method: adapting the spatial
support of the neighborhood according to local curvature can provide a
genuine geometric advantage for nearest neighbor classification.

A particularly interesting observation is that the obtained results support the interpretation that the main contribution of CARSANN lies not in selecting a more suitable value of $k$, but in introducing a different mechanism for defining locality: the \emph{spatial support} of the neighborhood is
adapted according to the local geometric complexity of the data. In this
sense, CARSANN extends the adaptive nearest-neighbor paradigm from
cardinality adaptation to \emph{geometry-driven support adaptation},
providing a complementary perspective in which curvature determines how
far local evidence should extend rather than simply how many observations
should be retained.

From a computational perspective, CARSANN introduces additional
preprocessing overhead associated with intrinsic-dimensionality estimation
and local mean-curvature computation. However, the geometric analysis is
performed in the intrinsic representation and the curvature estimates for
the training data are computed only once during classifier construction.
Under efficient nearest-neighbor search, the principal computational
scaling is $O(nd\log n) + O(nd^4)$, where $d$ is the intrinsic dimension.
Consequently, the proposed formulation can be substantially more
computationally favorable than performing the corresponding geometric
operations directly in a high-dimensional ambient space, particularly
when $d \ll m$. The separation between offline geometric preprocessing and
online classification also makes the additional computational cost
explicit and manageable.

Despite these encouraging results, several limitations remain and point
to interesting directions for future research. First, the current
framework relies on mean curvature as a scalar descriptor of local
geometric complexity. Future work could investigate richer geometric
signals, including the individual principal curvatures, curvature
anisotropy, or combinations of intrinsic and extrinsic geometric
quantities, to determine whether they can provide more informative
neighborhood adaptation. Second, the current curvature-to-radius mapping
could be replaced by a data-driven or learned transformation, potentially
allowing the relationship between geometric complexity and neighborhood
scale to be learned jointly with the classification objective. Third,
future studies could investigate more robust curvature estimators and
approximate nearest-neighbor structures for large-scale datasets, as well
as extensions to settings in which the manifold structure is noisy,
heterogeneous, or only locally valid. Finally, an especially promising
direction is to develop supervised or semi-supervised variants in which
geometric information is combined with class-discriminative signals,
potentially yielding neighborhood-selection mechanisms that jointly
account for manifold geometry, local density, and decision-boundary
structure.

\section*{Statements and declarations}

\subsection*{Funding}
Alexandre L. M. Levada thanks CNPq (National Council for Scientific and Technological Development) for the financial support through the grant \#301432/2025-2. This study is also partially funded by the Coordination for the Improvement of Higher Education Personnel (CAPES), Brazil - Finance Code 001.


\subsection*{Code availability}
Python scripts to reproduce the results reported in this paper may be found at \url{https://github.com/alexandrelevada/carsann}. The implementation is a first version of the CARSANN classifier.

\subsection*{Data availability}
All datasets used in the experiments are publicly available at \url{www.openml.org}.

\bibliography{main}

@book{Manfredo,
author = "Manfredo P. do Carmo",
title = "Differential Geometry of Curves and Surfaces",
publisher = "Dover Publications Inc.",
edition = "2nd",
address = "New York",
year = "2017",
}

@book{Tristan,
author = "Tristan Needham",
title = "Visual Differential Geometry and Forms: A Mathematical Drama in Five Acts",
publisher = "Princeton University Press",
year = "2021",
address = "Princeton",
}

@book{DGApp,
author = "John Oprea",
title = "Differential Geometry and its Applications",
publisher = "The Mathematical Association of America",
year = "2007",
edition = "2",
address = "New York",
}

@article {HessianEig,
    author = {Donoho, D. L. and Grimes, C.},
    title = {Hessian eigenmaps: Locally linear embedding techniques for high-dimensional data},
    volume = {100},
    number = {10},
    pages = {5591--5596},
    year = {2003},
    publisher = {National Academy of Sciences},
    journal = {Proceedings of the National Academy of Sciences}
}

@article{1,
    title = {Adaptive k-nearest neighbor classifier based on the local estimation of the shape operator},
    author = {A. L. M. Levada and F. Nielsen and M. F. C. Haddad},
    journal = {arXiv.org},
    year = {2024},
    doi = {10.48550/arXiv.2409.05084},
}

@InProceedings{2,
author="Ougiaroglou, S. and Nanopoulos, A. and Papadopoulos, A. N. and Manolopoulos, Y. and Welzer-Druzovec, T.",
editor="Ioannidis, Y. and Novikov, B. and Rachev, B.",
title="Adaptive k-Nearest-Neighbor Classification Using a Dynamic Number of Nearest Neighbors",
booktitle="Advances in Databases and Information Systems",
year="2007",
publisher="Springer Berlin Heidelberg",
address="Berlin, Heidelberg",
pages="66--82",
}

@article{3,
    title = {On the Use of Modified Adaptive Nearest Neighbors for Classification},
    author = {J. Maeng and Sungwan Bang and M. Jhun},
    journal = {The Korean Journal of Applied Statistics},
    volume = {23},
    number = {6},
    year = {2010},
    pages = {1093-1102},
    doi = {10.5351/KJAS.2010.23.6.1093},
}

@inproceedings{4,
author = {Kibanov, M. and Becker, M. and Mueller, J. and Atzmueller, M. and Hotho, A. and Stumme, G.},
title = {Adaptive kNN using expected accuracy for classification of geo-spatial data},
year = {2018},
publisher = {Association for Computing Machinery},
address = {New York, NY, USA},
doi = {10.1145/3167132.3167226},
booktitle = {Proceedings of the 33rd Annual ACM Symposium on Applied Computing},
pages = {857–865},
numpages = {9},
location = {Pau, France},
series = {SAC '18}
}

@INPROCEEDINGS{5,
  author={Sun, S. and Huang, R.},
  booktitle={2010 Seventh International Conference on Fuzzy Systems and Knowledge Discovery}, 
  title={An adaptive k-nearest neighbor algorithm}, 
  year={2010},
  volume={1},
  number={},
  pages={91-94},
  doi={10.1109/FSKD.2010.5569740}
}

@article{6,
  title={Classification Algorithm Based on Natural Nearest Neighbor},
  author={Q. Zhu and Y. Zhang and H. Liu},
  journal={The Journal of Information and Computational Science},
  year={2015},
  volume={12},
  pages={573-580},
  doi = {10.12733/JICS20105267},
}

@article{7,
    title = {An Enhanced Adaptive k-Nearest Neighbor Classifier Using Simulated Annealing},
    author = {A. Onyezewe and A. Kana and F. Abdullahi and A. O. Abdulsalami},
    journal = {International Journal of Intelligent Systems and Applications},
    year = {2021},
    doi = {10.5815/IJISA.2021.01.03},
    pages = {34-44}
}

@article{8,
author = {Zhang, J. and Bian, Z. and Wang, S.},
title = {Bayes-Decisive Linear KNN with Adaptive Nearest Neighbors},
journal = {International Journal of Intelligent Systems},
volume = {2024},
number = {1},
pages = {6664942},
doi = {https://doi.org/10.1155/2024/6664942},
year = {2024}
}

@INPROCEEDINGS{9,
  author={Zheng, L. and Zhu, K. and Zhao, T. and Zhou, D. and Deng, X. and Zhao, Y.},
  booktitle={2022 China Automation Congress (CAC)}, 
  title={Seabed sediments classification based on side-scan sonar images with adaptive feature fusion and metric learning}, 
  year={2022},
  volume={},
  number={},
  pages={1890-1895},
  doi={10.1109/CAC57257.2022.10055663}
}

@article{10,
    title = {Adaptive k-Nearest Centroid Neighbor Classifier for Detecting Drifted Twitter Spam},
    author = {L. Lalitha and V. R. Hulipalled},
    journal = {International Journal of Engineering and Advanced Technology},
    year = {2019},
    doi = {10.35940/ijeat.e1048.0585s19},
    volume = {8},
    number = {5S},
    pages = {235-242}
}

@article{11,
AUTHOR = {Sui, G. and Yan, J. and Wu, Y. and Xu, Z. and Qi, M. and Zhang, Z.},
TITLE = {Mechanical Fault Diagnosis of High-Voltage Circuit Breakers with Dynamic Multi-Attention Graph Convolutional Networks Based on Adaptive Graph Construction},
JOURNAL = {Applied Sciences},
VOLUME = {14},
YEAR = {2024},
NUMBER = {10},
ARTICLE-NUMBER = {4036},
DOI = {10.3390/app14104036}
}

@article{12,
AUTHOR = {Saadatfar, H. and Khosravi, S. and Joloudari, J. H. and Mosavi, A. and Shamshirband, S.},
TITLE = {A New K-Nearest Neighbors Classifier for Big Data Based on Efficient Data Pruning},
JOURNAL = {Mathematics},
VOLUME = {8},
YEAR = {2020},
NUMBER = {2},
ARTICLE-NUMBER = {286},
DOI = {10.3390/math8020286}
}

@article{13,
author = {Hu, Y.and Peng, G. and Wang, Z. and Cui, Y. and Qin, H.},
title = {Partition Selection for Large-Scale Data Management Using KNN Join Processing},
journal = {Mathematical Problems in Engineering},
volume = {2020},
number = {1},
pages = {7898230},
doi = {https://doi.org/10.1155/2020/7898230},
year = {2020}
}

@ARTICLE{14,
  author={Gong, C. and Demmel, J. and You, Y.},
  journal={IEEE Transactions on Knowledge and Data Engineering}, 
  title={Distributed and Joint Evidential K-Nearest Neighbor Classification}, 
  year={2024},
  volume={36},
  number={11},
  pages={5972-5985},
  doi={10.1109/TKDE.2023.3341098}
}

@article{CoverHart1967,
  author  = {Cover, Thomas M. and Hart, Peter E.},
  title   = {Nearest Neighbor Pattern Classification},
  journal = {IEEE Transactions on Information Theory},
  volume  = {13},
  number  = {1},
  pages   = {21--27},
  year    = {1967},
  doi     = {10.1109/TIT.1967.1053964}
}

@inproceedings{SunHuang2010,
  author    = {Sun, Shiliang and Huang, Rongqing},
  title     = {An Adaptive $k$-Nearest Neighbor Algorithm},
  booktitle = {Proceedings of the Seventh International Conference on Fuzzy Systems and Knowledge Discovery},
  pages     = {91--94},
  year      = {2010},
  doi       = {10.1109/FSKD.2010.5569740}
}

@article{HastieTibshirani1996,
  author  = {Hastie, Trevor and Tibshirani, Robert},
  title   = {Discriminant Adaptive Nearest Neighbor Classification},
  journal = {IEEE Transactions on Pattern Analysis and Machine Intelligence},
  volume  = {18},
  number  = {6},
  pages   = {607--616},
  year    = {1996},
  doi     = {10.1109/34.506411}
}

@article{FaccoEtAl2017,
  author  = {Facco, Elena and d'Errico, Maria and Rodriguez, Alex and Laio, Alessandro},
  title   = {Estimating the Intrinsic Dimension of Datasets by a Minimal Neighborhood Information},
  journal = {Scientific Reports},
  volume  = {7},
  number  = {1},
  pages   = {12140},
  year    = {2017},
  doi     = {10.1038/s41598-017-11873-y}
}

@article{LevadaNielsenHaddad2024,
  author  = {Levada, A. L. M. and Nielsen, F. and Haddad, M. F. C.},
  title   = {Adaptive $k$-Nearest Neighbor Classifier Based on the Local Estimation of the Shape Operator},
  journal = {arXiv preprint arXiv:2409.05084},
  year    = {2024},
  doi     = {10.48550/arXiv.2409.05084}
}

@article{Vanschoren2013,
  author  = {Vanschoren, Joaquin and Van Rijn, Jan N. and Bischl, Bernd
             and Torgo, Lu{\'\i}s},
  title   = {OpenML: Networked Science in Machine Learning},
  journal = {SIGKDD Explorations},
  volume  = {15},
  number  = {2},
  pages   = {49--60},
  year    = {2013},
  doi     = {10.1145/2641190.2641198}
}

\end{document}